\documentclass[journal]{IEEEtran}

\usepackage[T1]{fontenc}% optional T1 font encoding
\usepackage{cite}
\usepackage{amsmath,amssymb,amsfonts}
\usepackage[vlined,boxed,ruled,linesnumbered]{algorithm2e}
\usepackage{graphicx}
\usepackage{subfigure}
\usepackage{textcomp}
\usepackage{xcolor}
\usepackage{bbding}
\usepackage{booktabs}
\usepackage[flushleft]{threeparttable}
\usepackage{ntheorem}
\usepackage{multirow}
\usepackage{boxedminipage}
\usepackage{url}
\usepackage{diagbox}
\usepackage{booktabs}
\usepackage{stmaryrd} %空心方括号

\theoremseparator{.}%定理后的句号
\newtheorem{theorem}{Theorem}
\newtheorem{lemma}{Lemma}
\newtheorem{definition}{Definition}
\newtheorem*{proof}{Proof}
\newcommand{\lb}{\llbracket}
\newcommand{\rb}{\rrbracket}

\DeclareMathOperator{\lcm}{lcm}
\DeclareMathOperator{\st}{s.t. }
\renewcommand{\KwData}{\textbf{Input: }}
\renewcommand{\KwResult}{\textbf{Output: }}
\newcommand{\FuncName}[1]{\mbox{\normalfont\textsc{#1}}}

\colorlet{mix}{blue!58!black}

\ifCLASSINFOpdf

\else

\fi
\usepackage{amsmath}

\begin{document}
	
%	\title{PEGA: A Privacy-Preserving Genetic Algorithm for Combinatorial Optimization}
\title{Evolution as a Service: A Privacy-Preserving Genetic Algorithm for Combinatorial Optimization}
	
	\author{Bowen~Zhao, ~\IEEEmembership{Member,~IEEE,}
		Wei-Neng~Chen,~\IEEEmembership{Senior Member,~IEEE,}
		Feng-Feng~Wei,~\IEEEmembership{Student Member,~IEEE,}
		Ximeng~Liu,~\IEEEmembership{Senior Member,~IEEE,}
		Qingqi~Pei,~\IEEEmembership{Senior Member,~IEEE,}
		and Jun Zhang,~\IEEEmembership{Fellow,~IEEE}
		% <-this % stops a space
		\thanks{This work is supported in part by the Key Project of Science and Technology Innovation 2030 supported by the Ministry of Science and Technology of China (Grant No. 2018AAA0101300), the National Natural Science Foundation of China (Grant Nos. 62072109, U1804263, 61702105, and 61632013), and the Peng Cheng Laboratory Project of Guangdong Province PCL2018KP004.}
		\thanks{B. Zhao and Q. Pei are with Guangzhou Institute of Technology, Xidian University, Guangzhou, China. E-mail: bwinzhao@gmail.com, qqpei@mail.xidian.edu.cn}
		\thanks{W.-N. Chen and F.-F. Wei are with the School of Computer Science and Engineering, South China University of Technology, Guangzhou, China. E-mail: cwnraul634@aliyun.com}% <-this % stops a space
		\thanks{X. Liu is with the College of Computer and Data Science, Fuzhou University, Fujian, China, and Cyberspace Security Research Center, Peng Cheng Laboratory, Shenzhen. snbnix@gmail.com}% <-this % stops a space
		\thanks{J. Zhang is with the Department of Electrical and Electronic Engineering, Hanyang University, Ansan 15588, South Korea. E-mail: junzhang@ieee.org}
		\thanks{Corresponding authors: W.-N. Chen, B. Zhao}
		\thanks{Manuscript received xx xx, xxxx; revised xx xx, xxxx.}}
	
	\markboth{IEEE Transactions on Cybernetics,~Vol.~xx, No.~x, xx~xxxx}%
	{Shell \MakeLowercase{\textit{et al.}}:}
	
	\maketitle
	
\begin{abstract}
	Evolutionary algorithms (EAs), such as the genetic algorithm (GA), offer an elegant way to handle combinatorial optimization problems (COPs). However, limited by expertise and resources, most users do not have enough capability to implement EAs to solve COPs. An intuitive and promising solution is to outsource evolutionary operations to a cloud server, whilst it suffers from privacy concerns. To this end, this paper proposes a novel computing paradigm, evolution as a service (EaaS), where a cloud server renders evolutionary computation services for users without sacrificing users' privacy. Inspired by the idea of EaaS, this paper designs PEGA, a novel privacy-preserving GA for COPs. Specifically, PEGA enables users outsourcing COPs to the cloud server holding a competitive GA and approximating the optimal solution in a privacy-preserving manner. PEGA features the following characteristics. First, any user without expertise and enough resources can solve her COPs. Second, PEGA does not leak contents of optimization problems, i.e., users' privacy. Third, PEGA has the same capability as the conventional GA to approximate the optimal solution. We implements PEGA falling in a twin-server architecture and evaluates it in the traveling salesman problem (TSP, a widely known COP). Particularly, we utilize encryption cryptography to protect users' privacy and carefully design a suit of secure computing protocols to support evolutionary operators of GA on encrypted data. Privacy analysis demonstrates that PEGA does not disclose the contents of the COP to the cloud server. Experimental evaluation results on four TSP datasets show that PEGA is as effective as the conventional GA in approximating the optimal solution.
\end{abstract}
	
	
\IEEEpeerreviewmaketitle

\section{Introduction}
\IEEEPARstart{E}{volutionary} algorithms (EAs), such as genetic algorithm (GA), are powerful tools to tackle combinatorial optimization problems (COPs) in science and engineering fields. Many problems faced by science and engineering can be formulated as COPs, such as synthetic biology, transport of goods and planning, production planning \cite{naseri2020application,paschos2017applications}. EA has proven to be an powerful tool in handling COPs due to its global probabilistic search ability based on biological evolution such as selection and mutation \cite{wang2021set,radhakrishnan2021evolutionary}. Applications of science and engineering have strong requirements for EAs to tackle optimization problems \cite{paschos2017applications}.

Limited expertise and resources of common users hinder them from tackling COPs through EAs effectively. In practice, most users facing COPs lack expertise, such as EAs and programming skill for EAs implementation. Also, EAs based on biological evolution require plenty of iterative operations to search the approximate optimal solution, which consumes abundance computing resources. In the sight of users, even though they have the need for COPs, fail to effectively solve the COP due to limited capability and resources.

One promising and elegant solution is that the cloud server renders an evolutionary computing service for users. The cloud server is equipped with sufficient computing and storage resources and can offer convenient and flexible computation services, such as training and inference of machine learning \cite{gilad2016cryptonets,mohassel2018aby3,mishra2020delphi,rathee2020cryptflow2}, named machine learning as a service (MLaaS). In MLaaS, users outsource tasks of training or inference to the cloud server and get results. The cloud server performs computing of training or inference. As computing provided by the cloud server, MLaaS does not require users to have expertise and sufficient computing resources. Similarly, users are able to outsource tasks of evolutionary computation to the cloud server and get optimization results even though they lack programming skills for EAs implementation and sufficient resource to perform EAs.

Privacy concerns are critical challenges for outsourcing computation of EAs to the cloud server just like MLaaS \cite{gilad2016cryptonets,mohassel2018aby3,mishra2020delphi,rathee2020cryptflow2}. Optimization results of COPs are private information of users \cite{sakuma2007genetic,han2007privacy,funke2010privacy}. For example, optimization results of COPs for synthetic biology, transport of goods and planning, production planning involve private biologic information, planning of goods transportation and production to name but a few. The cloud server is not generally regarded as a trusted entity in an outsourcing computation scenario \cite{shan2018practical,mohassel2018aby3,mishra2020delphi,rathee2020cryptflow2}, such as iCloud leaking celebrity photos, Amazon Web Services exposing Facebook user records. Obviously, no user or company is willing to reveal biologic information and planning of goods transportation and production to others. Moreover, many regulations stipulate to protect personal data. GDPR\footnote{GDPR: General Data Protection Regulation (EU)} stipulates any information relating to an identified or identifiable natural person is private and should be protected. Also, contents of COPs should be regarded as private information. Given contents of the COP, the cloud server holding EAs can obtain the optimization results, which breaches privacy regulation.

To tackle privacy concerns of outsourcing computation of EAs, in this paper, we define a novel computing paradigm, evolution as a service (EaaS), the cloud server rendering evolutionary computing service for users without sacrificing users' privacy. Broadly speaking, the cloud server encapsulates EAs as a service. Users outsource tasks of evolutionary computation to the cloud server and adopt privacy-preserving methods (e.g., encryption cryptography) to protect privacy. The cloud server performs evolutionary operations and returns optimization results to users. In EaaS, the cloud server cannot learn users' contents of COP and optimization results. Also, users are not required to have expertise of EAs and sufficient resources. In short, EaaS enables users convenient and flexible solving COPs without sacrificing privacy, which relieves the dilemma between evolutionary computation and privacy concerns.
	
The vital idea of EaaS is that users outsource encrypted contents of the optimization problem to the cloud server, and the cloud server renders an evolutionary computation service for users over encrypted data. Technically, the implementation of EaaS suffers from several challenges.

\textbf{First}, the cloud server requires to perform evolutionary operations without sacrificing users' privacy. EA involves basic evolutionary operations including  population initialization, evaluation, selection, crossover, and mutation \cite{sakuma2007genetic,sun2020automatically}. The population initialization requires randomly generating several hundred or thousands of individuals, and each individual represents a possible solution. Arguably, when the cloud server has no knowledge about contents of the COP, it is not a trivial task to generate possible solutions. Furthermore, if the cloud server has difficulty in initializing the population, it is also challenging to perform evaluation, selection, crossover, and mutation operations as the latter relies on the initialized population.

\textbf{Second}, the cloud server can evaluate the fitness value of each individual in the population but fails to learn the possible solution. In EAs, the fitness value determines the quality of solutions and is a crucial metric to select dominant individuals. To protect users' privacy, it should prevent the cloud server from obtaining users' possible solutions \cite{sakuma2007genetic}. Unfortunately, if the cloud server has no knowledge of possible solutions, it fails to evaluate the fitness values of individuals in the population.

\textbf{Third}, the cloud server can select dominant individuals without knowing individuals' fitness values. EA is inspired by the process of natural selection, so its critical operation is to select dominant individuals based on individuals' fitness values. Technically, it requires the cloud server to compare individuals' fitness values under unknowing them. Intuitively, secure comparison protocols \cite{yao1982protocols,veugen2014encrypted,veugen2015secure} seems to provide a potential solution for this. However, the work \cite{yao1982protocols} requires two-party holding private data to perform a comparison operation. If the user participates in the comparison operation, it significantly increases the user's communication overhead as EA needs several hundred or thousands of individuals to generate the approximate optimal solution. The protocols \cite{veugen2014encrypted,veugen2015secure} only generates an encrypted comparison result. Given encrypted comparison results, the cloud server fails to select dominant individuals. In short, selecting dominant individuals has challenges in communications and operations.

To tackle the above challenges, this paper focuses on the implementation of EaaS through GA and carefully designs a privacy-preserving GA, called PEGA\footnote{PEGA comes form \textbf{\underline{P}}rivacy-pr\textbf{\underline{E}}serving \textbf{\underline{G}}enetic \textbf{\underline{A}}lgorithm}. Specifically, we exploit the threshold Paillier cryptosystem (THPC) \cite{lysyanskaya2001adaptive} and one-way mapping function to protect the user's privacy. The homomorphism of THPC enables evaluating individuals' fitness values over encrypted data. Also, we propose a suite of secure computation protocols to support privacy-preserving evolutionary operations of GA, such as selection. Our contributions can be concluded as three-folds.
\begin{itemize}
	\item We propose a novel computing paradigm, EaaS, a privacy-preserving evolutionary computation paradigm that outsources evolutionary computation to a cloud server. EaaS does not require users to have expertise of EAs and programming skills for EAs implementation but can output the approximate optimal solution for users. Furthermore, EaaS does not leak users' privacy to the cloud server.
	
	\item We carefully design PEGA, a privacy-preserving genetic algorithm based on the computing paradigm EaaS. Particularly, a secure division protocol (\texttt{SecDiv}) and a secure comparison protocol (\texttt{SecCmp}) are presented to support privacy-preserving fitness proportionate selection. \texttt{SecDiv} and \texttt{SecCmp} enable the cloud server computing the probability of each individual being selected and select potentially dominant individuals without disclosing possible solutions, respectively.
	
	\item We take four TSP (a widely kwnon COP) datasets (i.e., gr48, kroA100, eil101, and kroB200) to evaluate the effectiveness and efficiency of PEGA. Resultss of experiments and analyses on four TSP datasets demonstrate that PEGA is as effective as the conventional GA \cite{larranaga1999genetic} in approximating the optimal solution.
\end{itemize}

The rest of this paper is organized as follows. In Section II, the related work is briefly described. In Section III, we formulate EaaS and PEGA. The design of PEGA is elaborated in Section IV. In Section V, PEGA for TSP is given. Results of privacy analysis and experimental evaluation are shown in Section VI. Finally, we conclude the paper in Section VII.

\section{Related Work}
In this section, we briefly review privacy-preserving evolutionary algorithms (EAs). In contrast to privacy-preserving neural networks (NNs) inference \cite{gilad2016cryptonets,rathee2020cryptflow2}, privacy-preserving evolutionary algorithms have received little attention. One possible reason is EAs require the server to perform random operations, such as population initialization, mutation, while the operations of NNs are generally deterministic. Also, privacy-preserving NNs inference does not need the server to obtain intermediate results. On the contrary, privacy-preserving EAs require the server to learn intermediate results to perform subsequent operations. For example, the server requires to learn the plaintext comparison result to select dominant individuals.

Sakuma \textit{et al.} \cite{sakuma2007genetic} proposed a privacy-preserving GA by means of the idea of secure multi-party computation and the Paillier cryptosystem to solve TSP. The work \cite{sakuma2007genetic} considers a scenario where multiple servers hold traveling costs while a user wants to choose the server that provides the optimal route; Servers and the user are unwilling to disclose their own private data. Thus, the work \cite{sakuma2007genetic} requires interaction between the user and servers. Han \textit{et al.} \cite{han2007privacy} presented a privacy-preserving GA for rule discovery, where two parties holding datasets jointly perform a privacy-preserving GA to discover a better set of rules in a privacy-preserving manner. The scheme \cite{han2007privacy} also needs two parties to interact to generate an optimal solution. Funke \textit{et al.} \cite{funke2010privacy} designed a privacy-preserving multi-object EA based on Yao's secure two-party protocol \cite{yao1982protocols}. The authors in \cite{funke2010privacy} claim that their solution improves security and efficiency, but their solution still requires two parties to interact. Jiang \textit{et al.} \cite{jiang2020privacy} put forward to a cloud-based privacy-preserving GA by means of somewhat homomorphic encryption, where a user outsources operations of GA to the cloud server. However, the work \cite{jiang2020privacy} fails to support privacy-preserving selection operations, and no practical problem is involved to evaluate its effectiveness and efficiency. Zhan \textit{et al.} \cite{zhan2021new} proposed a rank-based cryptographic function (RCF) to construct privacy-preserving EAs including particle swarm optimization and differential evolution. However, the authors do not the construct of RCF and their scheme suffers from some privacy concerns. Although a designer in \cite{zhan2021new} fails to obtain the fitness function, he holds possible solutions. Thus, as long as the designer learns which solution is dominant, he can obtain the approximate optimal solution, which discloses a user's privacy.

From the view of existing privacy-preserving EAs, there is no effective solution that provides a privacy-preserving evolution service for users that does not require the user to interact. Motivated by this, we formulate EaaS and give its implementation through GA.

\section{Formulation}
In this section, we give formal definitions of evolutionary as a service (EaaS) and privacy-preserving genetic algorithm (PEGA), where PEGA is a concrete implementation of EaaS.
\subsection{Formulation for EaaS}
\begin{figure}
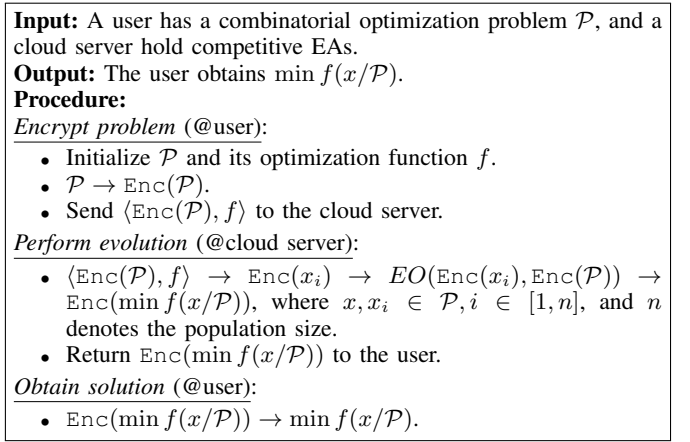

	\centering
	\small
	\begin{boxedminipage}{\linewidth}
		\textbf{Input:} A user has a combinatorial optimization problem $\mathcal{P}$, and a cloud server hold competitive EAs.\\
		\textbf{Output:} The user obtains $\min f(x/\mathcal{P})$.\\
		\textbf{Procedure:}\\
		\underline{\textit{Encrypt problem} (@user)}:
		\begin{itemize}
			\item Initialize $\mathcal{P}$ and its optimization function $f$.
			\item $\mathcal{P} \rightarrow \texttt{Enc}(\mathcal{P})$.
			\item Send $\langle \texttt{Enc}(\mathcal{P}), f\rangle$ to the cloud server.
		\end{itemize}
		\underline{\textit{Perform evolution} (@cloud server)}:
		\begin{itemize}
			\item $\langle\texttt{Enc}(\mathcal{P}), f\rangle \rightarrow \texttt{Enc}(x_i) \rightarrow EO(\texttt{Enc}(x_i),\texttt{Enc}(\mathcal{P})) \rightarrow \texttt{Enc}(\min f(x/\mathcal{P}))$, where $x,x_i\in\mathcal{P}, i\in[1, n]$, and $n$ denotes the population size.
			\item  Return $\texttt{Enc}(\min f(x/\mathcal{P}))$ to the user.
		\end{itemize}
		\underline{\textit{Obtain solution} (@user)}:
		\begin{itemize}
			\item $\texttt{Enc}(\min f(x/\mathcal{P})) \rightarrow \min f(x/\mathcal{P})$.
		\end{itemize}
	\end{boxedminipage}
	\caption{Structure of EaaS.}
	\label{s4eaas}
\end{figure}
\begin{definition}[EaaS]
	EaaS consists of users and a cloud server, where users have a requirement of solving a COP (denoted by $\mathcal{P}$) through evolutionary algorithms (EAs), whilst the cloud server holds competitive EAs and sufficient resources to perform EAs. The cloud server encapsulates EAs as a server and renders convenient and flexible evolutionary computing service for users. To avoid exposing privacy to the cloud server, users encrypt the content of the COP denoted by $\texttt{Enc}(\mathcal{P})$ and outsource it to the cloud server. Taking as input $\texttt{Enc}(\mathcal{P})$ and an EA, the cloud server performs evolutionary operations (e.g., evaluation, selection, crossover, mutation) denoted by $EO(\texttt{Enc}(x), \texttt{Enc}(\mathcal{P}))$ and returns an encrypted optimal solution $\texttt{Enc}(\min f(x/\mathcal{P}))$ to the user, where $\texttt{Enc}(x)$ indicates an encrypted solution, and $f(\cdot)$ is the objective function of $\mathcal{P}$. Formally, EaaS can be formulated as the following pattern
	\begin{align}
		\begin{aligned}
			\langle\texttt{Enc}(\mathcal{P}), f\rangle &\rightarrow \texttt{Enc}(x_i) \rightarrow EO(\texttt{Enc}(x_i), \texttt{Enc}(\mathcal{P}), f)\\
			& \rightarrow \texttt{Enc}(\min f(x/\mathcal{P})),\\
			&\st x,x_i\in \mathcal{P},\; i\in[1, n], 
		\end{aligned} 
	\end{align}
	where $n$ is the population size. Fig. \ref{s4eaas} shows the structure of EaaS.
	\label{def:eaas}
\end{definition}

%This paper formulates EaaS through encryption cryptography. Specifically, users formulate a combinatorial optimization problem $\mathcal{P}$ and its objective function $f$. Formally,
%\begin{align}
%	\begin{aligned}
%		&\min f(x/\mathcal{P}),\\
%		&\st x\in\mathcal{P},
%	\end{aligned}
%\end{align} 
%where $x$ indicates a possible solution. To protect privacy, users encrypt $\mathcal{P}$ as $\texttt{Enc}(\mathcal{P})$ and submit $\texttt{Enc}(\mathcal{P})$ and $f$ to a cloud server. Note that $\texttt{Enc}(\mathcal{P})$ is to encrypt the content of $\mathcal{P}$ rather than the type of $\mathcal{P}$. Next, the cloud server performs evolutionary operations on encrypted data, i.e., $EO(\texttt{Enc}(x_i),\texttt{Enc}(\mathcal{P}))$ ($i\in[1, n]$), where $EO(\cdot)$ means evolutionary operations, and $n$ is the population size. After $t$ iterations, the cloud server returns the encrypted optimal result $\texttt{Enc}(\min f(x/\mathcal{P}))$ to users. Formally, EaaS can be denoted by
%\begin{align}
%	\begin{aligned}
%		&\texttt{Enc}(\min f(x/\mathcal{P}))\leftarrow EO(\texttt{Enc}(x_i),\texttt{Enc}(\mathcal{P}),f),\\
%		&\st x,x_i\in\mathcal{P},\; i\in[1, n].
%	\end{aligned}
%\end{align}
%The workflow of EaaS shown in Fig. \ref{fig:eaas}.
%\begin{figure}
%	\centering
%	\includegraphics[width=0.82\linewidth]{figs/eaas.pdf}
%	\caption{Workflow of EaaS.}
%	\label{fig:eaas}
%\end{figure}
 
From Definition \ref{def:eaas} and Fig. \ref{s4eaas}, we see that EaaS does not ask the user to have the expertise and resources to solve a COP through EAs. To approximate the optimal solution of the COP, the user outsources operations to the cloud server. The cloud server is given encrypted data, so it fails to learn contents of the COP. In other words, EaaS enables the cloud server perform evolutionary operations over encrypted data and generates encrypted optimization solutions to protect the user's privacy. The key of EaaS is to support evolutionary operations on encrypted data. 

\subsection{Formulation for PEGA}
To validate the computing paradigm of EaaS, we take GA, a widely known EA, as an example to concrete EaaS, called PEGA. GA usually comprises 5 polynomial time operations: population initialization, evaluation, selection, crossover, and mutation, where the later four are regarded as evolutionary operators \cite{larranaga1999genetic}. A formal definition of PEGA can be as follow.
\begin{definition}[PEGA]
	A privacy-preserving genetic algorithm (PEGA) takes as input an encrypted COP $\texttt{Enc}(\mathcal{P})$ and its optimization function $f$, and outputs an encrypted optimization solution $\texttt{Enc}(\min f(x/\mathcal{P}))$, $\st x\in\mathcal{P}$. Formally, PEGA can be formulated as the following pattern
	\begin{align}
		\begin{aligned}
			\langle\texttt{Enc}(\mathcal{P}),f\rangle&\rightarrow\mathtt{I}\rightarrow\mathtt{E}\rightarrow[\mathtt{S}\rightarrow\mathtt{C}\rightarrow\mathtt{M}\rightarrow\mathtt{E}]\times t\\
			&\rightarrow \texttt{Enc}(\min f(x/\mathcal{P})),\\
			&\st x\in\mathcal{P},
		\end{aligned}	
	\end{align}
	where $\times$ indicates the repetition, and $t$ denotes the iteration times. $\mathtt{I}$, $\mathtt{E}$, $\mathtt{S}$, $\mathtt{C}$, and $\mathtt{M}$ indicate operations of population initialization, evaluation, selection, crossover, and mutation, respectively. Note that $\mathtt{I}$, $\mathtt{E}$, $\mathtt{S}$, $\mathtt{C}$, and $\mathtt{M}$ take as input encrypted data and output encrypted data.
	\label{def:pega}
\end{definition}
\begin{figure}
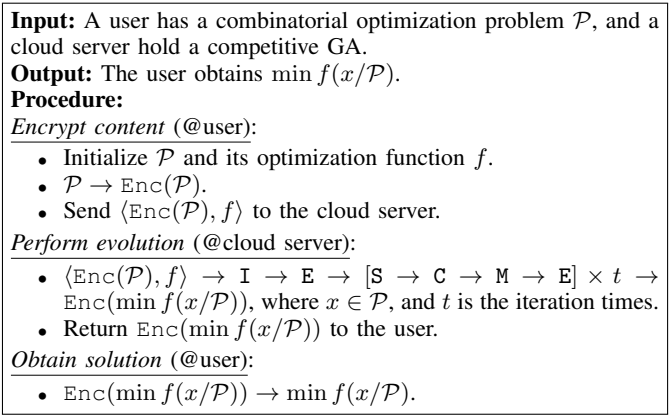

	\centering
	\small
	\begin{boxedminipage}{\linewidth}
		\textbf{Input:} A user has a combinatorial optimization problem $\mathcal{P}$, and a cloud server hold a competitive GA.\\
		\textbf{Output:} The user obtains $\min f(x/\mathcal{P})$.\\
		\textbf{Procedure:}\\
		\underline{\textit{Encrypt content} (@user)}:
		\begin{itemize}
			\item Initialize $\mathcal{P}$ and its optimization function $f$.
			\item $\mathcal{P} \rightarrow \texttt{Enc}(\mathcal{P})$.
			\item Send $\langle \texttt{Enc}(\mathcal{P}), f\rangle$ to the cloud server.
		\end{itemize}
		\underline{\textit{Perform evolution} (@cloud server)}:
		\begin{itemize}
			\item $\langle\texttt{Enc}(\mathcal{P}), f\rangle \rightarrow \mathtt{I} \rightarrow \mathtt{E} \rightarrow [\mathtt{S} \rightarrow \mathtt{C} \rightarrow \mathtt{M} \rightarrow \mathtt{E}]\times t \rightarrow \texttt{Enc}(\min f(x/\mathcal{P}))$, where $x\in\mathcal{P}$, and $t$ is the iteration times.
			\item  Return $\texttt{Enc}(\min f(x/\mathcal{P}))$ to the user.
		\end{itemize}
		\underline{\textit{Obtain solution} (@user)}:
		\begin{itemize}
			\item $\texttt{Enc}(\min f(x/\mathcal{P})) \rightarrow \min f(x/\mathcal{P})$.
		\end{itemize}
	\end{boxedminipage}
	\caption{Structure of PEGA.}
	\label{s4pega}
\end{figure}

From Definition \ref{def:eaas} and Definition \ref{def:pega}, we see that PEGA is to concrete $EO(\cdot)$ as $\mathtt{E}$, $\mathtt{S}$, $\mathtt{C}$, and $\mathtt{M}$. In next section, we elaborate on the design of PEGA, specially for how to execute evolutionary operations on encrypted data. 

\section{PEGA Design}
To self-contained, we first list threshold Paillier cryptosystem (THPC) used to encrypt the COP, and then give system model and threat model of PEGA. Next, details of PEGA design are illustrated.
\subsection{Primitive}
The detailed algorithms of THPC with (2, 2)-threshold decryption are listed as follows.

Key Generation (\texttt{KeyGen}): Let $\mathsf{p}=2\mathsf{p}'+1$ and $\mathsf{q}=2\mathsf{q}'+1$ be two big prime numbers with $\kappa$ bits (e.g., $\kappa=512$), where $\mathsf{p}',\mathsf{q}'$ are also prime numbers. The public key is denoted by $pk=(\mathsf{g}, N)$, where $N=\mathsf{p}\mathsf{q}$ and $\mathsf{g}=N+1$. The private key is denoted by $sk=(\lambda, \mu)$, where $\lambda=\lcm(\mathsf{p}-1, \mathsf{q}-1)$ and $\mu=\lambda^{-1}\mod N$. Particularly, the private key is split into $\lambda_1$ and $\lambda_2$ two partially private keys, $\st$, $\lambda_1+\lambda_2=0\mod\lambda$ and $\lambda_1+\lambda_2=1\mod N$. As $\mu=\lambda^{-1}\mod N$, $\lambda\mu=0\mod\lambda$ and $\lambda\mu=1\mod N$. Let $\lambda_1$ be a random integer in the interval $(0, \lambda N)$ and $\lambda_2=\lambda\mu-\lambda_1\mod\lambda N$.

Encryption (\texttt{Enc}): Take as input a message $m\in\mathbb{Z}_N$ and $pk$, and output $\lb m\rb\leftarrow\texttt{Enc}(pk, m)=(1+mN)\cdot r^N\mod N^2$, where $\lb m\rb=\lb m\mod N\rb$ and $r$ is a random number in $\mathbb{Z}_N^*$.

Decryption (\texttt{Dec}): Take as input a ciphertext $\lb m\rb$ and $sk$, and output $m\leftarrow\texttt{Dec}(sk, \lb m\rb)=L(\lb m\rb^\lambda\mod N^2)\cdot\mu\mod N$, where $L(x)=\frac{x-1}{N}$.

Partial Decryption (\texttt{PDec}): Take as input a cihpertext $\lb m\rb$ and a partially private key $\lambda_i$ ($i=$1 or  2), and output $M_i\leftarrow\texttt{PDec}(\lambda_i, \lb m\rb)=\lb m\rb^{\lambda_i}\mod N^2$.

Threshold Decryption (\texttt{TDec}): Take as input partially decrypted ciphtexts $M_1$ and $M_2$, and output $m\leftarrow\texttt{TDec}(M_1, M_2)=L(M_1\cdot M_2\mod N^2)$. 

The homomorphic operations on ciphertexts supported by THPC are described as follows.
\begin{itemize}
	\item Additive homomorphism: $\texttt{Dec}(sk, \lb m_1+m_2\mod N\rb)=\texttt{Dec}(sk, \lb m_1\rb\cdot\lb m_2\rb)$;
	\item Scalar-multiplication homomorphism: $\texttt{Dec}(sk, \lb c\cdot m\!\!\mod\! N\rb)$
	$=\texttt{Dec}(sk, \lb m\rb^c)$ for $c\in\mathbb{Z}_N$.
\end{itemize}
On the basis of additive homomorphism and scalar-multiplication homomorphism, THPC enables subtraction over encrypted data. Specifically, $\texttt{Dec}(sk, \lb m_1-m_2\rb)=\texttt{Dec}(sk, \lb m_1\rb\cdot\lb m_2\rb^{-1})$.

Note that any single partially private key fails to decrypt any ciphertexts. Also, as operations over ciphertexts encrypted by \texttt{Enc} require to perform a $\!\!\!\mod N^2$ operation, for brevity, we will omit $\!\!\!\mod N^2$ in the rest of this paper. Just like the Paillier cryptosystem \cite{paillier1999public}, THPC only works on integer. To effectively handle floating-point numbers, a given floating-point number $x$ be encoded as $\frac{x^\uparrow}{x^\downarrow}\cdot2^\ell$, where $\ell$ is a constant, for example, $\ell=53$ is used to encode a 64-bit floating-point number. In this paper, if a message $m$ to be encrypted is a floating-point number, it is encrypted as $\lb\frac{m^\uparrow}{m^\downarrow}\cdot2^\ell\rb$, i.e., $\lb  m\rb=\lb\frac{m^\uparrow}{m^\downarrow}\cdot2^\ell\rb$. To simplify notation, we use $\lb m\rb$ to denote $\lb\frac{m^\uparrow}{m^\downarrow}\cdot2^\ell\rb$ in the rest of paper.

\subsection{System Model and Threat Model}
In our system, we consider a user outsources an encrypted COP to twin-cloud servers (i.e., $S_1$ and $S_2$). Twin-cloud servers jointly provide a privacy-preserving GA service to solve the encrypted COP through performing secure two-party computations. The user obtains an encrypted optimization solution from $S_1$. As depicted in Fig. \ref{fig:sysmod}, FEGA comprises a user and twin cloud servers.
\begin{figure}
	\centering
	\includegraphics[width=0.81\linewidth]{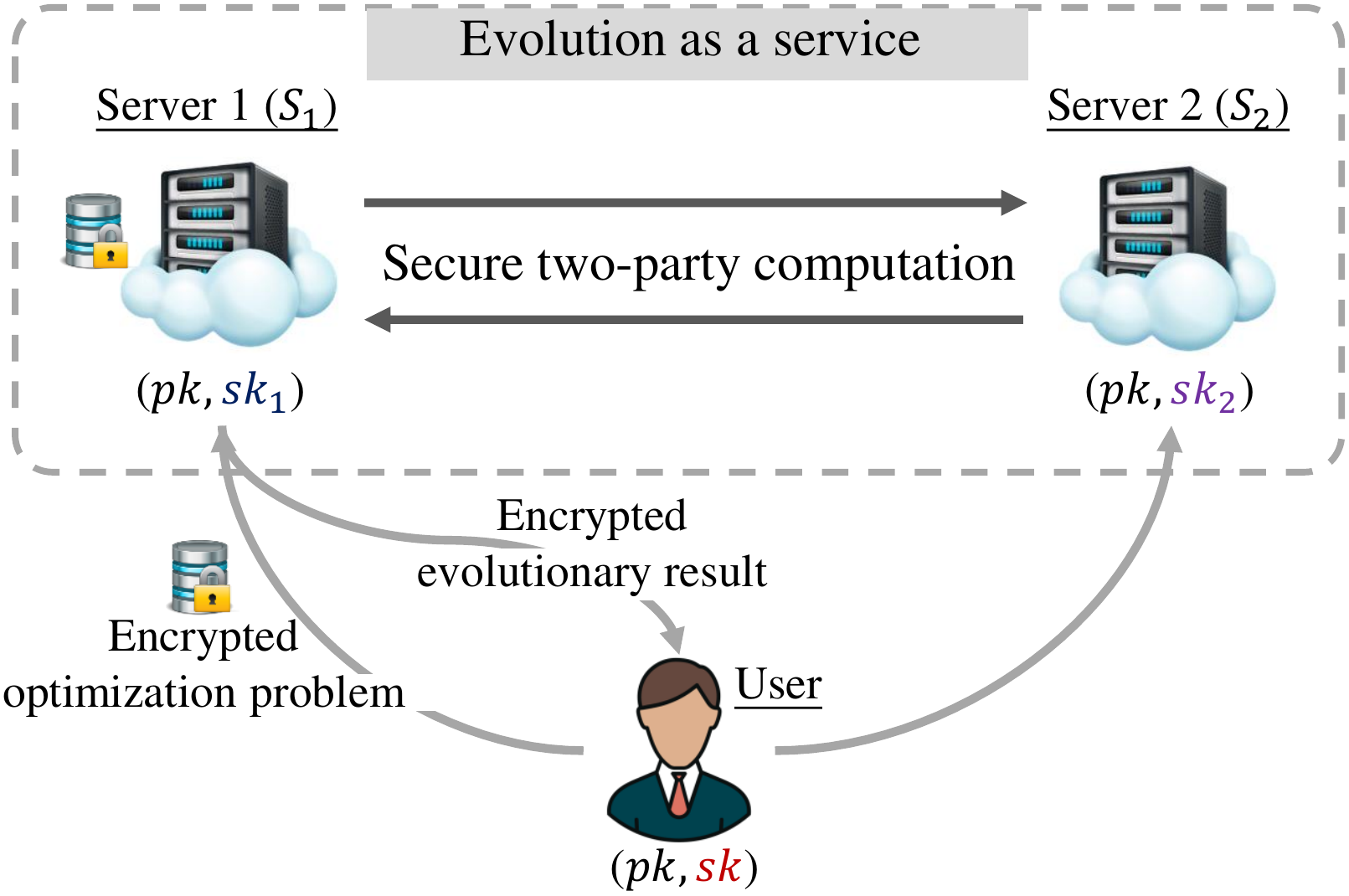}
	\caption{PEGA system model.}
	\label{fig:sysmod}
\end{figure}
\begin{itemize}
	\item \textbf{User}: The user has a COP $\mathcal{P}$ to be solved and outsources the problem to cloud servers with powerful computation and sufficient resources. To protect privacy, the user initializes a public/private pair ($pk, sk$) of THPC, and then encrypts the problem with the public key $pk$ as $\lb\mathcal{P}\rb$. Also, in order to enable cloud servers performing evolutionary operators over encrypted data, the user splits the private key $sk$ into two partially private keys $sk_1$ and $sk_2$ and sends them into $S_1$ and $S_2$, respectively.
	
	\item \textbf{Cloud server 1 ($S_1$)}: $S_1$ takes charge of storing $\lb\mathcal{P}\rb$ sent from the user. $S_1$ and $S_2$ jointly perform secure two-party computation protocols over encrypted data to support the operations of GA. Note that $S_1$ can directly execute certain homomorphic operations (e.g., additive homomorphism and scalar-multiplication homomorphism) over encrypted data supported by THPC.
	
	\item \textbf{Cloud server 2 ($S_2$)}: $S_2$ is responsible for assisting $S_1$ to perform the operations of GA in a privacy-preserving manner.
\end{itemize}

In the system of PEGA, the computation is outsourced to cloud servers. According to the outsourced computation situation \cite{shan2018practical}, there is one type of adversary that attempts to obtain the user's private information, i.e., contents of the COP. The adversary involves either $S_1$ or $S_2$. Inspired by prior work \cite{rathee2020cryptflow2,mohassel2017secureml}, we assume either $S_1$ and $S_2$ are \textit{curious-but-honest} (or say \textit{semi-honest}), i.e., they follow required computation protocols and perform required computations correctly, but may try to obtain the user's private information with the help of encrypted TSP and intermediate computation results. Note that $S_1$ and $S_2$ do not share their partially private keys and parameters in a non-colluding twin-server architecture \cite{rathee2020cryptflow2,mohassel2017secureml}. The assumption of no-colluding twin-cloud servers is reasonable. Anyone cloud server shares the private parameters or intermediate computation results with the other one, which means the cloud gives data leakage evidence to the other one. Arguably, for its own commercial interests, any cloud server is unwilling to provide data leakage evidence to other.

\subsection{Overview of PEGA}
In this section, we give a high-level description of our proposed PEGA. The goal of PEGA is to perform operations of GA over encrypted data and output an encrypted optimization solution. As shown in Fig. \ref{fig:overview}, PEGA consists of 5 polynomial-time operations, i.e., \textsc{Gen\_Initial\_Pop}, \textsc{Evaluation}, \textsc{Selection}, \textsc{Crossover}, and \textsc{Mutation}, and their briefly description is given as follows.
\begin{figure*}
	\centering
	\includegraphics[width=0.9\linewidth]{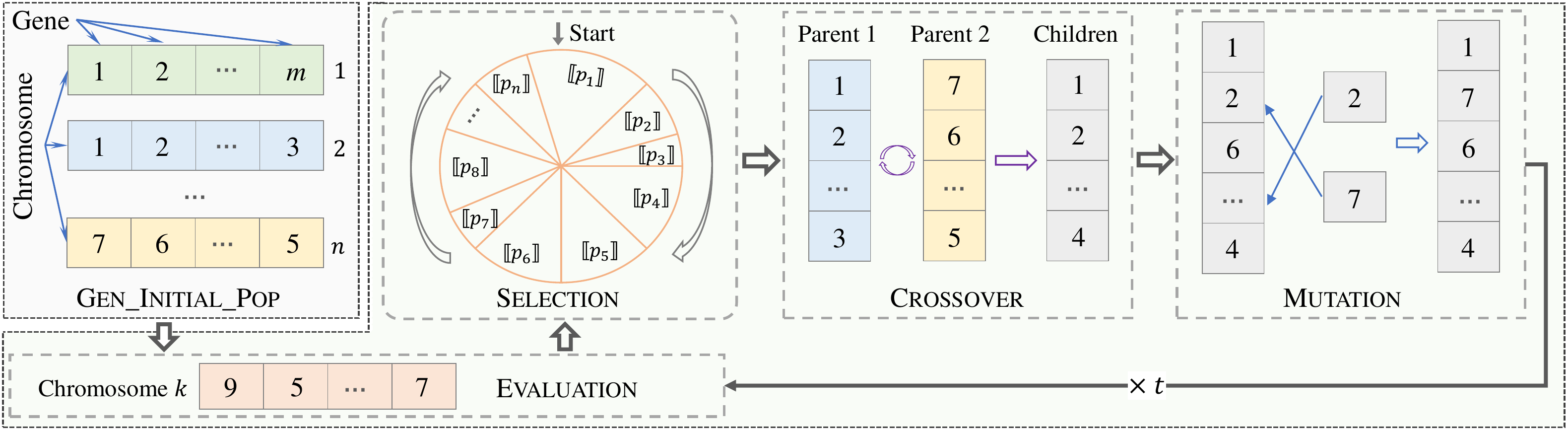}
	\caption{Overview of PEGA.}
	\label{fig:overview}
\end{figure*}
\begin{itemize}
	\item \textsc{Gen\_Initial\_Pop}: Given an encrypted COP $\lb\mathcal{P}\rb$, \textsc{Gen\_Initial\_Pop} randomly generates a population compromising $n$ individuals. Each individual is denoted by a chromosome. Each chromosome consists of $m$ genes. Formally, \textsc{Gen\_Initial\_Pop} takes as input $\lb\mathcal{P}\rb$, and outputs $n$ encrypted chromosomes denoted by $\{\lb x_1\rb, \cdots, \lb x_n\rb\}$, $\st x_i\in\mathcal{P}$ and $i\in[1, n]$.
	
	\item \textsc{Evaluation}: Given $n$ encrypted chromosomes, \textsc{Evaluation} firstly computes the fitness value of each encrypted chromosome. Specifically, \textsc{Evaluation} utilizes the homomorphism of THPC to obtain encrypted fitness value of each encrypted chromosome according to the optimization function $f$. Formally, \textsc{Evaluation} takes as input $\{\lb x_1\rb,\cdots, \lb x_n \rb\}$ and $f$, and outputs $\{\lb f(x_1)\rb, \cdots, \lb f(x_n)\rb\}$. Also, \textsc{Evaluation} outputs the optimal chromosome holding minimum fitness value. To this end, we carefully design a secure comparison protocol (\texttt{SecCmp}) that can compare $\lb f(x_i) \rb$ and $\lb f(x_j) \rb$ ($i,j\in[1, n]$ and $i \neq j$). Formally, given $\lb f(x_i) \rb$ and $\lb f(x_j) \rb$, \texttt{SecCmp} outputs $\lb f(x_i) \rb$, where $f(x_i) \leq f(x_j)$. Thus, given $\{\lb f(x_1) \rb, \cdots, \lb f(x_n) \rb\}$, \textsc{Evaluation} can output $\lb f(x_k) \rb$ via \texttt{SecCmp}, where $\lb f(x_k)\rb$, $\st f(x_k)=\min\{f(x_1,), \cdots, f(x_n)\}$ and $k\in[1, n]$.
	
	\item \textsc{Selection}: For $n$ encrypted chromosomes, \textsc{Selection} explores the well-studied fitness proportionate selection operator \cite{zhong2005comparison} to select dominant individuals. Specifically, \textsc{Selection} firstly computes an encrypted probability for each individual. After that, \textsc{Selection} performs operations of fitness proportionate selection over encrypted probabilities. The critical operations of \textsc{Selection} include addition, division, and comparison on encrypted data. To enable division on encrypted data, we propose a secure division protocol (\texttt{SecDiv}). Formally, given $\lb a \rb$ and $\lb b \rb$, \texttt{SecDiv} outputs $\lb \frac{a}{b} \rb$.
	
	\item \textsc{Crossover}: Choosing two chromosomes as parents, \textsc{Crossover} performs a crossover operator to enable two parents crossover generating children. Roughly speaking, \textsc{Crossover} exchanges some genes of two parents to generate new chromosomes.
	
	\item \textsc{Mutation}: For each encrypted chromosome, \textsc{Mutation} exchanges some genes of the chromosome to generate a new chromosome.
\end{itemize}

When iterations is used as the termination condition, except for \textsc{Gen\_Initial\_Pop}, \textsc{Evaluation}, \textsc{Selection}, \textsc{Crossover}, and \textsc{Mutation} require to repeat. PEGA takes encrypted data as input, generates encrypted intermediate results, and outputs encrypted optimization solution to protect privacy.

\subsection{Privacy-preserving Protocols for PEGA}
In this section, we first elaborate on the secure division protocol (\texttt{SecDiv}) and the secure comparison protocol (\texttt{SecCmp}) that are used to construct PEGA. Next, through \texttt{SecDiv} and \texttt{SecCmp}, we design a secure probability algorithm (\texttt{SecPro}) and a secure fitness proportionate selection algorithm (\texttt{SecFPS}) to support \textsc{Selection} on encrypted data. Also, \texttt{SecCmp} enables \textsc{Evaluation} over encrypted data.

\subsubsection{\textbf{Secure Division Protocol (\texttt{SecDiv})}}
Given $\lb x\rb$ and $\lb y\rb$, where $y\neq 0$, \texttt{SecDiv} outputs $\lb\frac{x}{y}\rb$. The key idea of \texttt{SecDiv} is to convert division to scalar multiplication. Specifically, for any integers $x,y$, we have $\frac{x}{y}=\frac{x\cdot\mathtt{y}}{2^\ell}$, where $\mathtt{y}=\frac{1}{y}\cdot2^\ell$ and $\mathtt{y}$ is an integer. Formally, $\texttt{SecDiv}(\lb x\rb,\lb y\rb)\rightarrow\lb\frac{x}{y}\cdot2^\ell\rb$. \texttt{SecDiv} consists of three steps. 
\begin{itemize}
	\item [(1)] $S_1$ calls \texttt{PDec} to partially decrypt $\lb y\rb$ to get $[y]_1$. Next, $S_1$ sends $\langle\lb y\rb, [y]_1\rangle$ to $S_2$.
	
	\item [(2)] $S_2$ calls \texttt{PDec} to partially decrypt $\lb y\rb$ to get $[y]_2$ and then calls \texttt{TDec} to obtain $y$ with $[y]_1$ and $[y]_2$. After that, $S_2$ computes $\mathtt{den}=\frac{1}{y}$. Next, $S_2$ encodes $\mathtt{den}$ as $\mathtt{y}=\frac{\mathtt{den}^\uparrow}{\mathtt{den}^\downarrow}\cdot2^\ell$. Finally, $S_2$ returns $\mathtt{y}$ to $S_1$. Clearly, $\mathtt{y}$ is an integer.
	
	\item [(3)] $S_1$ computes $\lb\frac{x}{y}\rb\leftarrow\lb x\rb^\mathtt{y}$.
\end{itemize}

\subsubsection{\textbf{Secure Comparison Protocol (\texttt{SecCmp})}}
Given $\lb x\rb$ and $\lb y\rb$, \texttt{SecCmp} outputs 0 when $x\geq y$, 1 otherwise ($x<y$). Formally, $\texttt{SecCmp}(\lb x\rb,\lb y\rb)\rightarrow\{0,1\}$. \texttt{SecCmp} consists of three steps.
\begin{itemize}
	\item [(1)] $S_1$ generates a random number $\pi\in\{0,1\}$ through tossing a coin. $S_1$ computes
	\begin{equation}
		\lb\Delta\rb=\left\{
		\begin{aligned}
			&(\lb x\rb\cdot\lb y\rb^{-1})^{r_1}\cdot\lb r_1+r_2\rb, &\mbox{for $\pi=0$}\\
			&(\lb y\rb\cdot\lb x\rb^{-1})\cdot\lb r_2\rb, &\mbox{for $\pi=1$}
		\end{aligned}
		\right.
	\end{equation}
	where $r_1,r_2$ are two randomly integers, $\st$, $r_1\leftarrow\{0,1\}^\sigma\setminus\{0\}$, $r_1+r_2>\frac{N}{2}$, and $r_2\leq\frac{N}{2}$. $\sigma$ is secure parameter, e.g., $\sigma=128$. Next, $S_1$ calls \texttt{PDec} to get $[\Delta]_1$, and sends $\langle\lb\Delta\rb, [\Delta]_1\rangle$ to $S_2$.
	
	\item [(2)] $S_2$ calls \texttt{PDec} to get $[\Delta]_2$ and then calls $\texttt{TDec}$ to obtain $\Delta$ with $[\Delta]_1$ and $[\Delta]_2$. If $\Delta>\frac{N}{2}$, $S_2$ sets $u=0$, otherwise, $u=1$. Finally, $S_2$ sends $u$ to $S_1$.
	
	\item [(3)] $S_1$ obtains the comparison result by computing $\pi\oplus u$. When $\pi=0$, we have $\pi\oplus u=0$ when $u=0$, otherwise, $\pi\oplus u=1$. When $\pi=1$, we have $\pi\oplus u=1$ when $u=0$, otherwise, $\pi\oplus u=0$.
\end{itemize}

Clearly, given $\{\lb f(x_1)\rb,\cdots,\lb f(x_n)\rb\}$, it is easy to implement \textsc{Evaluation} by calling \texttt{SecCmp}. Specifically, \textsc{Evaluation} performs $n$ comparison operations on encrypted data to obtain the optimal chromosome holding minimum fitness values.

For brevity, we utilize $v_i$ to denote $f(x_i)$. According to the fitness proportionate selection operator \cite{zhong2005comparison}, it requires to compute each individual's probability. Thus, the individual's probability is denoted by
\begin{align}
	p_i\leftarrow\frac{v_i}{\sum_{j=1}^n v_j}.
	\label{eq:probability}
\end{align}
However, to protect users' privacy, the cloud server only obtains $\lb v_i\rb$. Given $\{\lb v_1\rb, \cdots, \lb v_n \rb\}$, it is not trivial for the cloud server to compute $\frac{\lb v_i\rb}{\lb \sum_{j=1}^n v_j\rb}$. Fortunately, the proposed \texttt{SecDiv} offers a potential solution. Specifically, through the proposed \texttt{SecDiv}, we design a secure probability algorithm (\texttt{SecPro}) to compute each individual's probability on encrypted data. Given $n$ encrypted fitness values $\{\lb v_1\rb,\cdots,\lb v_n\rb\}$, \texttt{SecPro} outputs $n$ encrypted probabilities $\{\lb p_1\rb,\cdots,\lb p_n\rb\}$. Formally, $\texttt{SecPro}(\{\lb v_1\rb,\cdots,\lb v_n\rb\})\rightarrow\{\lb p_1\rb,\cdots,\lb p_n\rb\}$. As shown in Algorithm \ref{spro}, \texttt{SecPro} consists of three steps.
\begin{itemize}
	\item [(1)] $S_1$ firstly computes $\lb sum \rb \leftarrow \Pi_{j=1}^n\lb v_j\rb$ by the additive homomorphism of THPC, so we have $sum \leftarrow \sum_{i=1}^n v_i$. After that, $S_1$ calls \texttt{PDec} to partially decrypt $\lb sum\rb$ to get $[sum]_1$. Next, $S_1$ sends $\langle\lb sum\rb, [sum]_1\rangle$ to $S_2$.
	
	\item [(2)] $S_2$ calls \texttt{PDec} to partially decrypt $\lb sum\rb$ to get $[sum]_2$ and then calls \texttt{TDec} to obtain $sum$ with $[sum]_1$ and $[sum]_2$. After that, $S_2$ computes $den=\frac{sum}{2^\ell}$ and $\mathtt{den}=\frac{1}{den}$. Next, $S_2$ encodes $\mathtt{den}$ as $\mathtt{D}=\frac{\mathtt{den}^\uparrow}{\mathtt{den}^\downarrow}\cdot2^\ell$. Finally, $S_2$ returns $\mathtt{D}$ to $S_1$. Clearly, $\mathtt{D}$ is an integer.
	
	\item [(3)] $S_1$ computes $\lb p_i\rb\leftarrow(\lb v_i\rb)^\mathtt{D}$ for $i\in[1, n]$. It can be seen that $p_i\leftarrow\frac{v_i}{sum} \cdot 2^\ell$.
\end{itemize}
\begin{algorithm}
	\SetAlgoLined
	\KwData{$S_1$ has $\{\lb v_1\rb,\cdots,\lb v_n\rb\}$.}\\
	\KwResult{$S_1$ obtains $\{\lb p_1\rb,\cdots,\lb p_n\rb\}$.}
	
	{Step 1. $S_1$  computes
		\begin{itemize}
			\item $\lb sum \rb \leftarrow \Pi_{j=1}^n \lb v_j \rb$;
			\item $[sum]_1\leftarrow\texttt{PDec}(sk_1, \lb sum\rb)$;
			\item and sends $\langle\lb sum\rb, [sum]_1\rangle$ to $S_2$.
	\end{itemize}}
	{Step 2. $S_2$ computes
		\begin{itemize}
			\item $[sum]_2\leftarrow\texttt{PDec}(sk_2, \lb sum \rb)$ and $sum\leftarrow\texttt{TDec}([sum]_1,[sum]_2)$;
			\item $den=\frac{sum}{2^\ell}$, $\mathtt{den}=\frac{1}{den}$, and $\mathtt{D}=\frac{\mathtt{den}^\uparrow}{\mathtt{den}^\downarrow}\cdot2^\ell$;
			\item and sends $\mathtt{D}$ to $S_1$.
	\end{itemize}}
	{Step 3. $S_1$ computes
		\begin{itemize}
			\item $\lb p_i\rb\leftarrow\lb v_i\rb^{\mathtt{D}}$ for $i\in[1, n]$.
	\end{itemize}}		
	\caption{$\texttt{SecPro}(\{\lb D_i\rb,\cdots,\lb D_n\rb\},\lb sum\rb)\rightarrow\{\lb p_1\rb,\cdots,\lb p_n\rb\}$.}
	\label{spro}
\end{algorithm}

From Algorithm \ref{spro}, we see that $\texttt{Dec}(sk, \lb p_i\rb)=\frac{v_i}{\sum_{j=1}^n v_j}\cdot2^{2\ell}$. If $\frac{v_i}{\sum_{j=1}^n v_j}>\frac{v_j}{\sum_{i=1}^n v_i}$, we have $\texttt{Dec}(sk, \lb p_i\rb)>\texttt{Dec}(sk, \lb p_j\rb)$. In other words, Algorithm \ref{spro} does not change the numerical relationship among probabilities of individuals.

Also, to enable fitness proportionate selection on encrypted data, we construct a secure fitness proportionate selection algorithm (\texttt{SecFPS}) via \texttt{SecCmp}. Given $n$ encrypted probabilities $\{\lb p_1\rb,\cdots,\lb p_n\rb\}$, \texttt{SecFPS} outputs $n$ individuals. The key idea of \texttt{SecFPS} is to perform $n$ comparison operations over encrypted data. Formally, $\texttt{SecFPS}(\lb p_1\rb,\cdots,\lb p_n\rb)\rightarrow Pop$, where $Pop$ represents a population consisting of $n$ individuals. As shown in Algorithm \ref{sfps}, \texttt{SecFPS} consists of three steps.
\begin{itemize}
	\item [(1)] $S_2$ generates $n$ encrypted random numbers $\{\lb r_1\rb,\cdots,\lb r_n\rb\}$ and sends them to $S_1$. Note that as $\texttt{Dec}(sk, \lb p_i\rb)=\frac{v_i}{\sum_{i=1}^n v_i}\cdot2^{2\ell}$, the random number $r_i$ multiplies by $2^{2\ell}$ to reach the same order of magnitude as $p_i$.
	
	\item [(2)] $S_1$ computes $\lb p_i\rb\leftarrow\lb p_{i-1}\rb\cdot\lb p_i\rb$ for $i\in[2, n]$. Thus, we have $p_i=\sum_{j=1}^i p_j$. In other words, $S_1$ produces a ciphertext set of orderly sequence $\{p_1,\cdots,\sum_{i=1}^n p_i\}$.
	
	\item [(3)] $S_1$ and $S_2$ jointly perform a binary search over encrypted data to find the individual $i$ $\st$, $r_j\leq p_i$ and $r_j>p_{i-1}$ ($i,j\in[1,n]$ through calling \texttt{SecCmp}. Repeat step (3) until generating $n$ individuals.
\end{itemize}
\begin{algorithm}
	\SetAlgoLined
	\KwData{$S_1$ has $\{\lb p_1\rb, \cdots,\lb p_n\rb\}$.}\\
	\KwResult{$S_1$ obtains $Pop$.}\\
	$S_2$ computes $\lb r_i\rb\leftarrow\texttt{Enc}(pk, r_i\cdot2^{2\ell})$ for $i\in[1, n]$, where $r_i$ is a random number in $(0,1)$ and then sends $\{\lb r_1\rb,\cdots,\lb r_n\rb\}$ to $S_1$;\\
	\For{$i=2$ to $n$}{
		$S_1$ computes $\lb p_i\rb\leftarrow\lb p_{i-1}\rb\cdot\lb p_i\rb$;
	}
	\For{$j=1$ to $n$}{
		$S_1$ and $S_2$ jointly perform $i\leftarrow\textsc{FindIndividual}(\lb r_j\rb, 1, n)$\;
		$S_1$ adds $i$ to $Pop$\;
	}
	\FuncName{FindIndividual}$(\lb r\rb, low, high)$ 
	\Begin{
		$i\leftarrow\lfloor\frac{low+high}{2}\rfloor$\;
		\If{$\texttt{SecCmp}(\lb p_i\rb, \lb r\rb)$ returns $1$}{
			\Return \textsc{FindIndividual}$(\lb r\rb, i+1, high)$\;
		}
		\Else{
			\If{$i<2$}{
				\Return $1$\;
			}
			\Else{
				\If{$\texttt{SecCmp}(\lb p_{i - 1}\rb, \lb r\rb)$ returns $1$}{
					\Return $i$\;
				}
				\Else{
					\Return \textsc{FindIndividual}$(\lb r\rb, low, i-1)$\;
				}
			}
		}
	}
	\caption{$\texttt{SecFPS}(\{\lb p_1\rb,\cdots,\lb p_n\rb\})\rightarrow Pop$.}
	\label{sfps}
\end{algorithm}

Note that the proposed \texttt{SecCmp} can be used to construct secure selection operators, such as secure tournament selection, secure elitism selection. The critical operation for tournament selection and elitism selection is to compare fitness values of individuals \cite{zhong2005comparison}, which is supported by \texttt{SecCmp}.

\section{PEGA for TSP}
This section takes TSP, a widely known COP, as an example to demonstrate the idea of EaaS through the proposed PEGA.

\subsection{Problem Encryption}
Given a list of cities and the traveling cost between each possible city pair, TSP is to find the shortest possible route that visits each city exactly once and returns to the origin city. Formally, as shown in Fig. \ref{fig:tsp}, the TSP can be denoted by a strictly upper triangular matrix, where the entry in the matrix represents the traveling cost of a city pair. For example, "6" is the traveling cost between WDC and CHI.
\begin{figure}
	\centering
	\includegraphics[width=\linewidth]{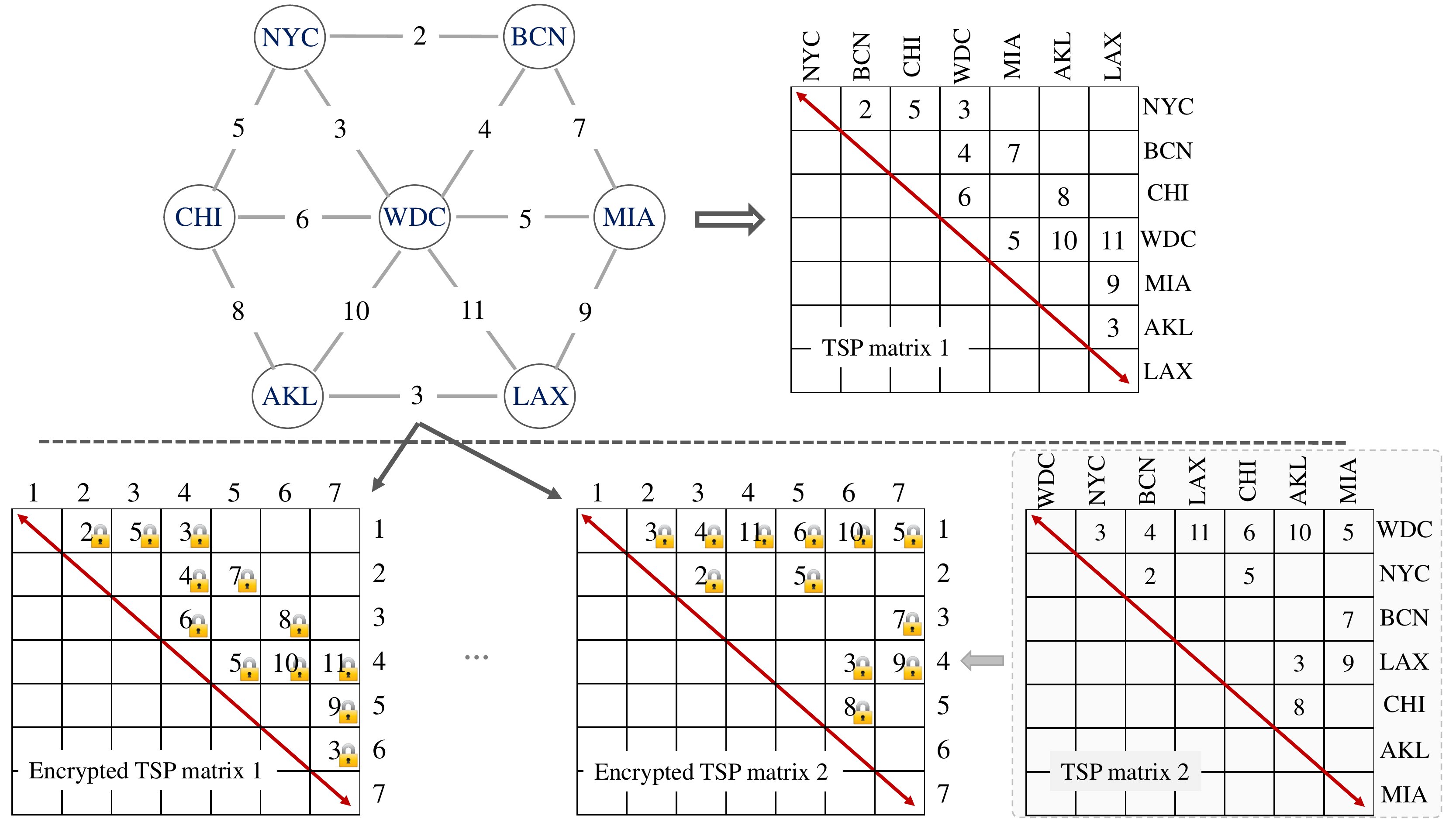}
	\caption{An example of encrypted TSP.}
	\label{fig:tsp}
\end{figure}

\begin{definition}[Encrypted TSP]
	An encrypted TSP means the city list and traveling cost between each possible city pair of a plaintext TSP are mapped into random numbers through cryptographical functions; it requires finding the shortest possible route that visits each city once and returns to the origin city over encrypted city list and traveling costs. Formally, let $\mathbf{M}$ be a TSP matrix, the encrypted TSP is denoted by 
	\begin{equation}
		\lb\mathbf{M}\rb\leftarrow F(\mathbf{M}),
	\end{equation}
	where $F(\cdot)$ represents a family of cryptographical functions.
\end{definition}

Clearly, generating an encrypted TSP requires encrypting the list of cities and the traveling cost between possible city pairs. On the one hand, as described PEGA, we exploit THPC to encrypt TSP. Specifically, each entry of $\mathbf{M}$ is encrypted through THPC. On the other hand, a one-way hash function can serve as the cryptographical function to map the city list into random numbers. However, a hash function $H(\cdot)$ always generates the same output when the same input is given. If the cloud server knows all cities, it is easy to obtain the city list of $\mathbf{M}$ through executing efficient hashing operations. Also, the output of $H(\cdot)$ is usually more than 256 bits, which incurs a high communication and storage cost. As depicted in Fig.\ref{fig:tsp}, we observe that given a TSP, its representation matrix is not unique. Inspired by this, we assume all cities are denoted by $\mathbb{C}$, and their mapping is denoted by $\mathbb{N}$, where $\mathbb{N}$ is the set of natural numbers. Thus, in this paper, we define a one-way function $\phi(\cdot)$ that randomly maps one city into one unique natural number. For example, in Fig. \ref{fig:tsp}, "WDS" is mapped into "4" and "1" in encrypted TSP 1 and encrypted TSP 2, respectively. Formally, for any item (e.g., "1") of $\mathbb{N}$, it can represent any city. Thus, when the city list of a TSP is randomly mapped into a natural number, the cloud server fails to obtain the city list.

Fig. \ref{fig:tsp} gives the storage structure of encrypted TSP. Specifically, the first row and the first column of $\lb\mathbf{M}\rb$ denotes the city index. $d_{i,j}$ represents the traveling cost between city $i$ and city $j$. "0" indicates that the two cities are unreachable, whereas it indicates two cities are reachable. Assume the size of cities be $m$, the objective function $f$ can be denoted by
\begin{align}
	f(x)=\underbrace{d_{i,j} + d_{j,k} + \cdots + d_{l, i}}_{m-1},
\end{align}
where $x$ is a possible route. Finally, a user outsources $\langle\lb \mathbf{M} \rb, f(\cdot)\rangle$ to $S_1$.

\subsection{Problem Solving via PEGA}
In this section, we elaborate on how to solve an encrypted TSP through PEGA.

\subsubsection{Initialization}
Given $\lb \mathbf{M} \rb$, $S_1$ initializes $n$ encrypted chromosomes denoted by $\{\lb x_1 \rb, \cdots, \lb x_n \rb\}$, where $\lb x_i \rb$ is denoted by an array of index of $\lb \mathbf{M} \rb$, such as "$9$-$5$-$\cdots$-$7$". As the one-way function $\phi(\cdot)$ is adopted, the index of $\lb \mathbf{M} \rb$ does not disclose the city. Thus, given $\lb \mathbf{M} \rb$, $S_1$ is able to generate $n$ encrypted chromosomes to initialize a population.

\subsubsection{Evaluation}
Given $\lb \mathbf{M} \rb$ and $\{\lb x_1 \rb, \cdots, \lb x_n \rb\}$, $S_1$ can compute $\{\lb f(x_1) \rb, \cdots, \lb f(x_n) \rb\}$. Specifically, without loss of generality, let $\lb x_i \rb$ be denoted by "$i$-$j$-$\cdots$-$k$", $S_1$ computes an encrypted fitness value $\lb v_i \rb$ as
\begin{align}
	\lb v_i \rb \leftarrow \underbrace{\lb d_{i,j} \rb \cdot \cdots \cdot \lb d_{k,i}\rb}_{m-1}.
\end{align}
As the additive homomorphsim of THPC, we see that $v_i=d_{i,j}+\cdots+d_{k,i}$. Thus, $S_1$ can calculate and generate $\{\lb v_1 \rb, \cdots, \lb v_n \rb\}$. Next, $S_1$ and $S_2$ jointly compute and find out the encrypted chromosome holding minimum fitness value by calling \texttt{SecCmp}. Specifically, without loss of generality, assume $v_i=\min\{v_1, \cdots, v_n\}$, i.e., $v_i\leq \{v_j\}_{j=1,j \neq i}^n$, $S_1$ outputs $\lb x_i \rb$ denoted by "$i$-$j$-$\cdots$-$k$" and sets it as the optimal chromosome.

\subsubsection{Selection}
$S_1$ can choose different selection operators, such as fitness proportionate selection, tournament selection, elitism selection, to perform a selection operator. In here, we consider $S_1$ utilizes the fitness proportionate selection operator as the selection operator. Specifically, $S_1$ firstly cooperates with $S_2$ to obtain $\{\lb p_1 \rb, \cdots, \lb p_n \rb\}$ by calling \texttt{SecPro}, where $p_i$ is the probability of the individual $i$ ($i\in[1, n]$). After that, $S_1$ teams with $S_2$ to generate a new population by calling \texttt{SecFPS}.

\subsubsection{Crossover}
Given encrypted chromosomes $\{\lb x_1 \rb, \cdots, \lb x_n \rb\}$, $S_1$ can adopt the conventional crossover operator (such as edge recombination crossover operator, ERX \cite{larranaga1999genetic}) to generate children. Assume $S_1$ chooses $\lb x_i \rb$ denoted by "$i$-$j$-$\cdots$-$k$" and $\lb x_j \rb$ denoted by "$j$-$i$-$\cdots$-$l$" as two parent chromosomes, it is easy for $S_1$ to generate two children by calling ERX \cite{larranaga1999genetic}. 

\subsubsection{Mutation}
Given encrypted chromosomes, $S_1$ is able to perform mutation operations on $\lb x_i \rb$. Specifically, $S_1$ can change the element of $\lb x_i \rb$.

\section{Results of Privacy Analysis and Experiment}
\subsection{Privacy Analysis}
THPC \cite{lysyanskaya2001adaptive} have been proved to be semantically secure. Thus, the homomorphic operations performed by $S_1$ do not disclose the user's private data. In this paper, we carefully design \texttt{SecDiv} and \texttt{SecCmp} based on a non-colluding twin-server architecture to select dominant individuals in a privacy-preserving manner. In this section, we demonstrate \texttt{SecDiv} and \texttt{SecCmp} are secure to perform division and comparison over encrypted data.

\begin{theorem}
	\label{tho:sdiv}
	Given $\lb x\rb$ and $\lb y \rb$, where $y\neq 0$, \texttt{SecDiv} does not disclose $\frac{x}{y}$.
\end{theorem}
\begin{proof}
	Given $\lb x\rb$ and $y$ ($y\neq 0$), \texttt{SecDiv} computes $\lb x\rb^{\frac{1}{y}\cdot2^\ell}$ to produce $\lb\frac{x}{y}\rb$. When $\ell$ is larger enough, $\frac{1}{y}\cdot2^\ell$ must be an integer. Without loss of the generality, let $c=\frac{1}{y}\cdot2^\ell$, we have $c\in\mathbb{Z}_N^*$. Thus, \texttt{SecDiv} essentially is to perform one scalar multiplication operation. As THPC is semantically secure, $\lb m\rb^c$ does not leak $cm$. Therefore, \texttt{SecDiv} does not disclose $\frac{x}{y}$.
\end{proof}

\begin{lemma}
	\texttt{SecPro} can produce each individual's encrypted fitness value and encrypted probability without leaking the individual's city list and route cost.
\end{lemma}
\begin{proof}
	According to Algorithm \ref{spro}, we see $\lb v_i\rb\leftarrow\texttt{SecDiv}(\lb sum\rb\cdot\lb D_i\rb^{-1}, sum)$, where $sum$ is the sum of all individuals' route costs, and $D_i$ is the individual $i$ route cost. As Theorem \ref{tho:sdiv} holds, \texttt{SecPro} can securely compute encrypted fitness values. Also, we have $\lb p_i\rb\leftarrow\texttt{SecDiv}(\lb v_i\rb, \lb \sum_{i=1}^{n} v_i \rb)$. Thus, we say that \texttt{SecPro} can securely compute encrypted probabilities when Theorem \ref{tho:sdiv} holds.
\end{proof}

\begin{theorem}
	\label{tho:scmp}
	Given $\lb x\rb$ and $\lb y\rb$, $\texttt{SecCmp}$ does not disclose $x$ and $y$.
\end{theorem}
\begin{proof}
	In the view of $S_1$, he only learns encrypted data, so \texttt{SecCmp} does not disclose $x$ and $y$ to $S_1$ as THPC is semantically secure. In the view of $S_2$, he can learns $r_1\cdot(x-y+1)+r_2$ ($\pi=0)$ or $r_1\cdot(y-x)+r_2$ ($\pi=1$). However, as $r_1$ and $r_2$ are unkown for $S_2$, given either $r_1\cdot(x-y+1)+r_2$ or $r_1(y-x)+r_2$, $S_2$ fails to get $x$, $y$, $x-y+1$, and $y-x$. Thus, \texttt{SecCmp} does not leak $x$ and $y$ to $S_2$. In short, \texttt{SecCmp} does not disclose $x$ and $y$. Furthermore, even though $S_2$ knows $y$, he cannot get $x$ as fails to know $r_1$ and $r_2$.
\end{proof}

\begin{lemma}
	\texttt{SecFPS} can select dominant individuals over encrypted data without leaking the individual's probability.
\end{lemma}
\begin{proof}
	From Algorithm \ref{sfps}, we see although $S_1$ can obtain $r_i\leq p_i$ or $r_i>p_i$ ($i\in[1,n]$), he fails to know $r_1$. Thus, we say that $S_1$ fails to learn $p_i$. Also, $S_2$ can get $r_1\cdot(r_i-p_i+1)+r_2$ or $r_1\cdot(p_i-r_i)+r_2$, but he fails to obtain $p_i$ as Theorem \ref{tho:scmp} holds. In short, \texttt{SecFPS} does not disclose the individual's probability.
\end{proof}

\subsection{Experimental Evaluation}
In this section, we evaluate the effectiveness of PEGA by comparing it with two conventional GA variants \cite{larranaga1999genetic} and give the performance of PEGA in terms of the computational complexity and communication costs. Specifically, the first GA variant adopts the fitness proportionate selection as the selection operator (named GA1), and the second one adopts the $k$-tournament as the selection operator (named GA2). Note that GA1 and GA2 utilize the ERX operator as the crossover operator due to its remarkable performance for TSP \cite{larranaga1999genetic}. Through our proposed secure computing protocols, PEGA can support fitness proportionate selection and $k$-tournament selection. Roughly speaking, \texttt{SecFPS} based on \texttt{SecCmp} and \texttt{SecDiv} enables the fitness proportionate selection. Also, \texttt{SecCmp} naturally supports the $k$-tournament selection. 

The experiments are executed on four most widely used TSP datasets\footnote{http://comopt.ifi.uni-heidelberg.de/software/TSPLIB95/tsp/} (i.e., gr48, kroA100, eil101, kroB200), where gr48 and kroA100 are small scale, while eil101 and kroB are medium scale \cite{wei2019empirical}. We implement PEGA and conventional GA variants in Java. The experiment is performed on a personal computer running windows 10-64bit with an Intel Core i7-4790 CPU \@ 3.6 GHz processor and 16 GB memory, which acts as the user. Also, the server running windows 10 64 bit with an Intel Core i7-10700 CPU \@ 2.9 GHz processor, and 32 GB memory, which simulates two cloud servers. Since GA is a stochastic approach, 30 independent runs are executed for each algorithm to generate an average. Experimental settings are listed in Table \ref{parm}, where $\|x\|$ denotes the length of $x$ in bits. The crossover rate and the mutation rate use settings in \cite{wei2019empirical,zhong2005comparison}.
\begin{table}
	\centering
	\begin{threeparttable}
		\caption{Experimental Parameter Settings}
		\begin{tabular}{rl}
			\toprule
			Parameters & Values                         \\ \midrule
			The population size $n$ & $n=300$                \\
			The crossover rate & $0.08, 0.1$ \\
			The mutation rate & $0.1, 0.15$                   \\
			$\ell$ & $\ell=106$                      \\
			$\|N\|$, $\|\lambda_1\|$ & $\|N\|=256$, $\|\lambda_1\|=80$                   \\
			$k$-tournament & $k=2$\\
			The max number of generations & 10000                            \\ \bottomrule
		\end{tabular}
		\label{parm}
	\end{threeparttable}
\end{table}

\subsection{Effectiveness Evaluation}
Given four TSPs, i.e., gr48, kroA100, eil101, and kroB200, we firstly compare the performance between GA1 and GA2. Experimental results are shown in Fig. \ref{fig:ga1vsga2}.  The $x-$axis is the number of generations and the $y-$axis is the path length of routing. The red solid line and blue dashed line represent GA1 and GA2, respectively. As depicted in Fig. \ref{fig:ga1vsga2}, we see that GA2 is remarkably superior to GA1 in terms of convergence. Specifically, in contrast to GA1, GA2 always converges to a smaller path length of routing in four given TSPs. In other words, GA2 has a stronger ability in approximating the optimal solution than GA1. Thus, we argue that $k$-tournament selection outperforms fitness proportionate selection for TSPs. One possible explanation is that $k$-tournament selection always selects a dominant individual into next generation, whilst poor individuals are possible to be selected by the fitness proportionate selection.
\begin{figure}[ht]
	\centering
	\subfigure[rossover rate (0.08), mutation rate (0.1)]{
		\includegraphics[width=0.46\linewidth]{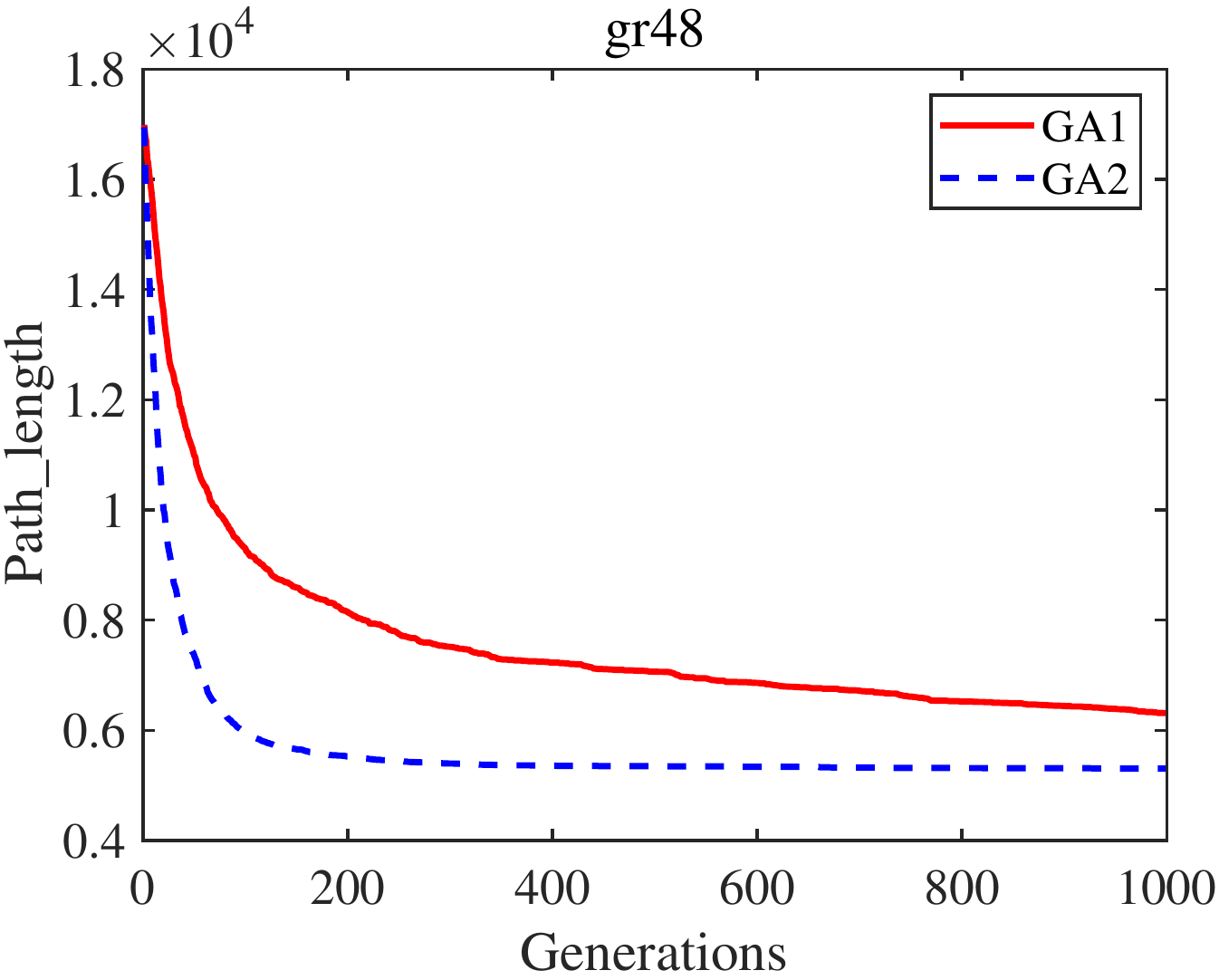}}
	\hspace{0.05in}
	\subfigure[rossover rate (0.1), mutation rate (0.15)]{
		\includegraphics[width=0.46\linewidth]{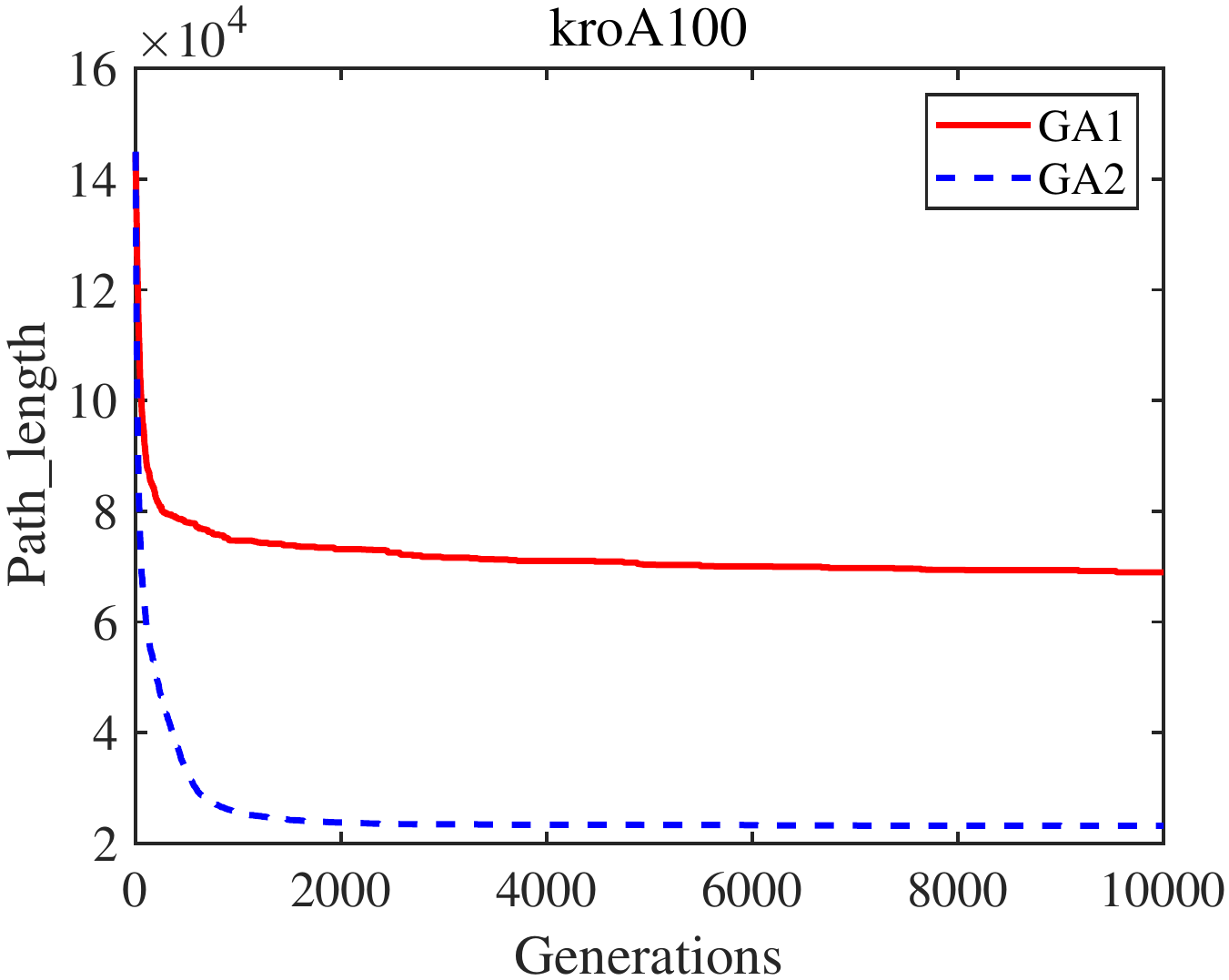}}
	\subfigure[rossover rate (0.1), mutation rate (0.15)]{
		\includegraphics[width=0.46\linewidth]{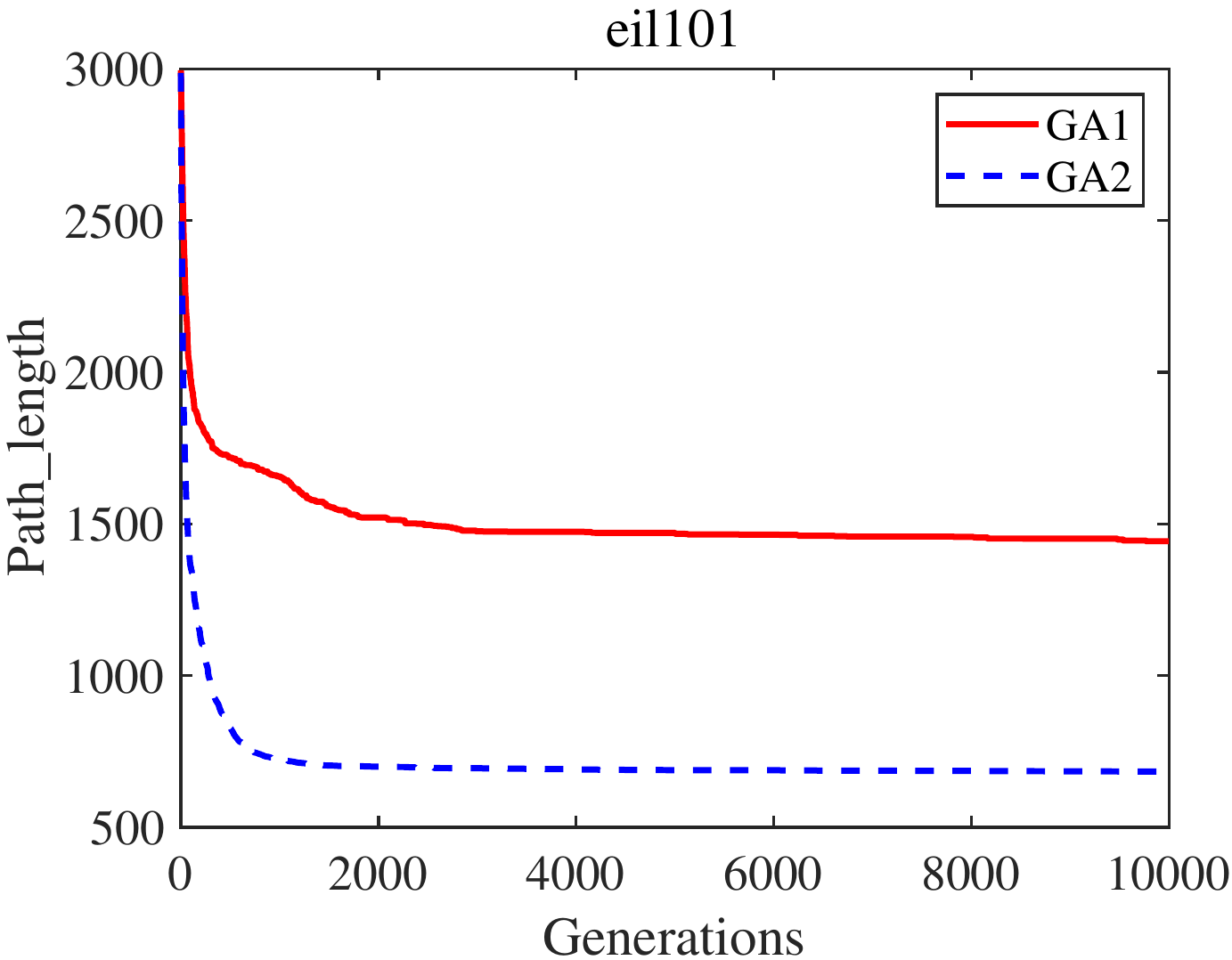}}
	\hspace{0.05in}
	\subfigure[rossover rate (0.1), mutation rate (0.15)]{
		\includegraphics[width=0.46\linewidth]{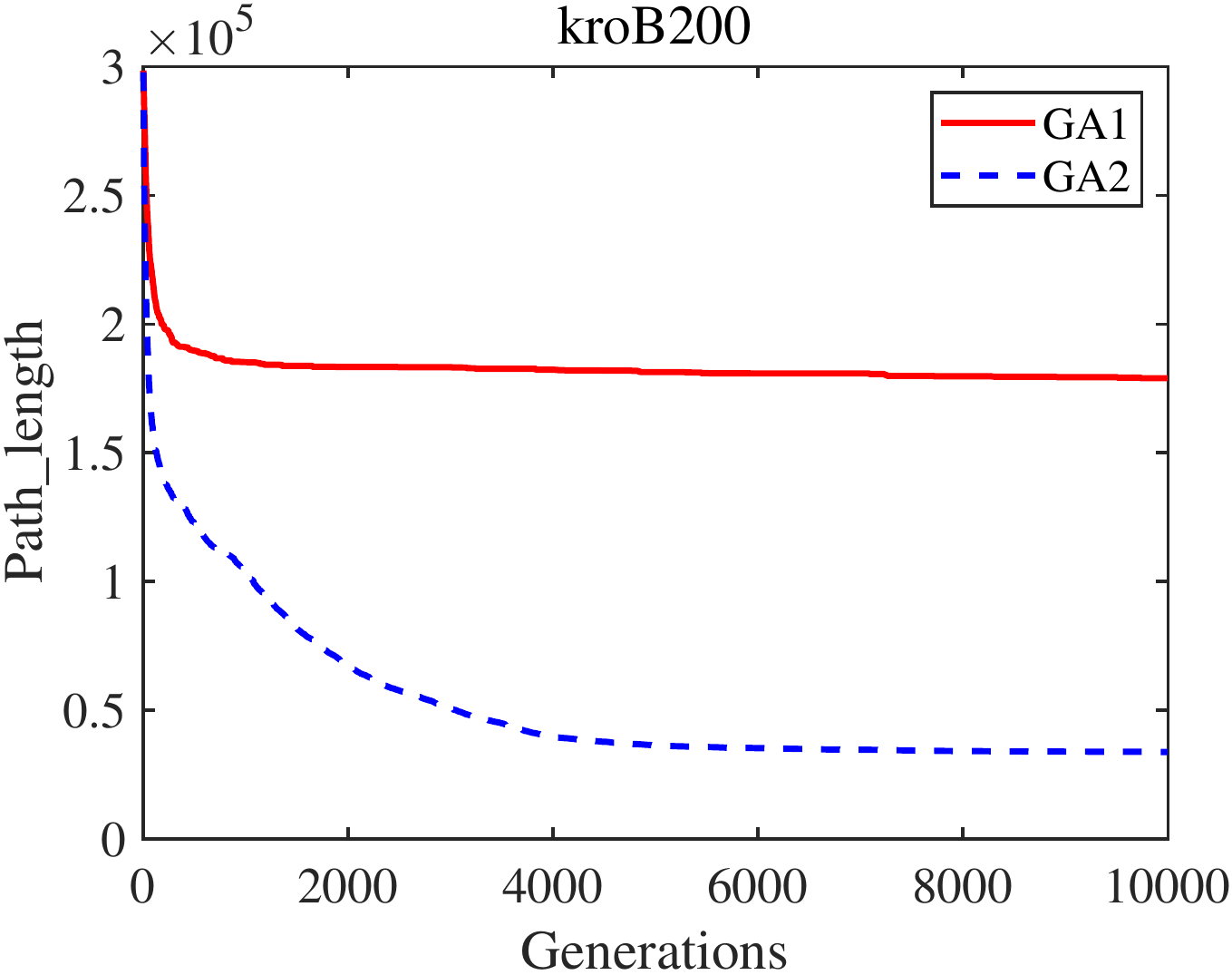}}
	\caption{Comparison of the convergence between GA1 and GA2 (The average result of 30 independent runs).}
	\label{fig:ga1vsga2}
\end{figure}

Although GA2 outperforms GA1 as shown in Fig. \ref{fig:ga1vsga2}, to demonstrate the effectiveness of the proposed PEGA, we construct PEGA1 and PEGA2, where PEGA1 and PEGA2 adopt the same evolutionary operators as GA1 and GA2, respectively. For four TSPs, i.e., gr48, kroA100, eil101, and kroB200, the comparison results between PEGA and GA are presented in Table \ref{table::comp}. To perform statistical tests, Wilcoxon rank-sum test at significance level 0.05 is adopted to examine whether compared results are significantly different. Also, mean and standard deviation are tested. The best results are highlighted in bold based on the p-value of the Wilcoxon rank-sum test. Particularly, to make a fair comparison, PEGA1 and PEGA2 use the same initial population as GA1 and GA2, respectively.

As depicted in Table \ref{table::comp}, in terms of mean, PEGA1 outperforms GA1 on gr48, kroA100, and kroB200. Meanwhile, PEGA2 outperforms GA2 on gr48 and kroA100. In terms of std, PEGA1 has less std on gr48, eil101, and kroB200 that are not exactly the same as those PEGA1 being superior on the mean. Thus, we can learn that less mean does not generate less std. From Table \ref{table::comp}, we see that the p-value in four TSPs is larger than 0.05, so it can conclude that there is no significant difference between PEGA1 and GA1. Similarly, PEGA2 and GA2 do not have significant difference. One possible explanation is that PEGA and GA perform the same evolution operators. Furthermore, our proposed secure computing protocols do not introduce noise into computational results, which guarantees calculation accuracy. The only difference between PEGA and GA is that PEGA performs evolution operators on encrypted data to protect privacy, on the contrary, GA performs evolution operators on cleartext data directly. The statistical results of mean and std between PEGA and GA are different. This is because PEGA and GA use different random numbers during performing evolution operators.
\begin{table*}[ht]
	\centering
	\caption{Comparison Results between PEGA and GA (The average result of 30 independent runs)}
	\begin{threeparttable}
		\begin{tabular}{c|c|c|c|c|c|c}
			\toprule
			Problems          & Scale                  &    Statistical tests     &      PEGA1     &     GA1    &PEGA2 & GA2 \\ \midrule
			\multirow{3}{*}{gr48} & \multirow{6}{*}{small} & mean    &     \textbf{6.2071e+03}      &     6.3138e+03     &  \textbf{5.2949e+03} & 5.3033e+03 \\
			&                        & std     &      \textbf{515.74}     &    530.22      & 119.78 & \textbf{98.32}\\
			&                        & p-value & \multicolumn{2}{c|}{\textbf{0.3183}} & \multicolumn{2}{c}{\textbf{1.0}} \\ \cmidrule(r){1-1} \cmidrule(l){3-7} 
			\multirow{3}{*}{kroA100} &      & mean    &   \textbf{6.8017e+04}        &   6.8961e+04       & \textbf{2.2819e+04} & 2.3175e+04 \\
			&                        & std     &      2.4004e+03     &   \textbf{1.9935e+03}       &\textbf{619.31} & 755.14\\
			&                        & p-value & \multicolumn{2}{c|}{\textbf{0.3615}} & \multicolumn{2}{c}{\textbf{0.1150}} \\\midrule 
			\multirow{3}{*}{eil101} &     \multirow{6}{*}{medium}                    & mean    &    1.4739e+03       &   \textbf{1.4431e+03}       & 686.4667 & \textbf{683.8667} \\
			&                        & std     &     \textbf{45.85}      &     82.73     & 15.68 & \textbf{9.99}\\
			&                        & p-value & \multicolumn{2}{c|}{\textbf{0.5475}} & \multicolumn{2}{c}{\textbf{0.5746}} \\\cmidrule(r){1-1} \cmidrule(l){3-7}
			\multirow{3}{*}{kroB200} &       & mean    &      \textbf{1.7723e+05}     &   1.7895e+05      & 3.3878e+04 &\textbf{3.3775e+04} \\
			&                        & std     &     5.8435e+03      &     \textbf{3.6344e+03 }    & 761.45 & \textbf{615.0}\\
			&                        & p-value & \multicolumn{2}{c|}{\textbf{0.4679}} & \multicolumn{2}{c}{\textbf{0.8419}} \\ \bottomrule
		\end{tabular}\label{table::comp}
	\end{threeparttable}
\end{table*}

To visualize the above conclusion, we plot convergence curves of PEGA1, GA1, PEGA2, and GA2 on four TSPs shown in Fig. \ref{fig:all}. The $x-$axis is the number of generations and the $y-$axis is the path length of routing. The red solid line and blue solid line represent GA1 and GA2, respectively. Cyan dashed line and black dashed line represent PEGA1 and PEGA2, respectively. Fig. \ref{fig:all} visually shows that PEGA1 and GA1 have the same convergence trend, and PEGA2 and GA2 has the same convergence trend. As shown in Table \ref{table::comp} and Fig. \ref{fig:all}, we argue that PEGA is as effective as GA for TSPs in approximating the optimal solution.
\begin{figure}[ht]
	\centering
	\subfigure[rossover rate (0.08), mutation rate (0.1)]{
		\includegraphics[width=0.46\linewidth]{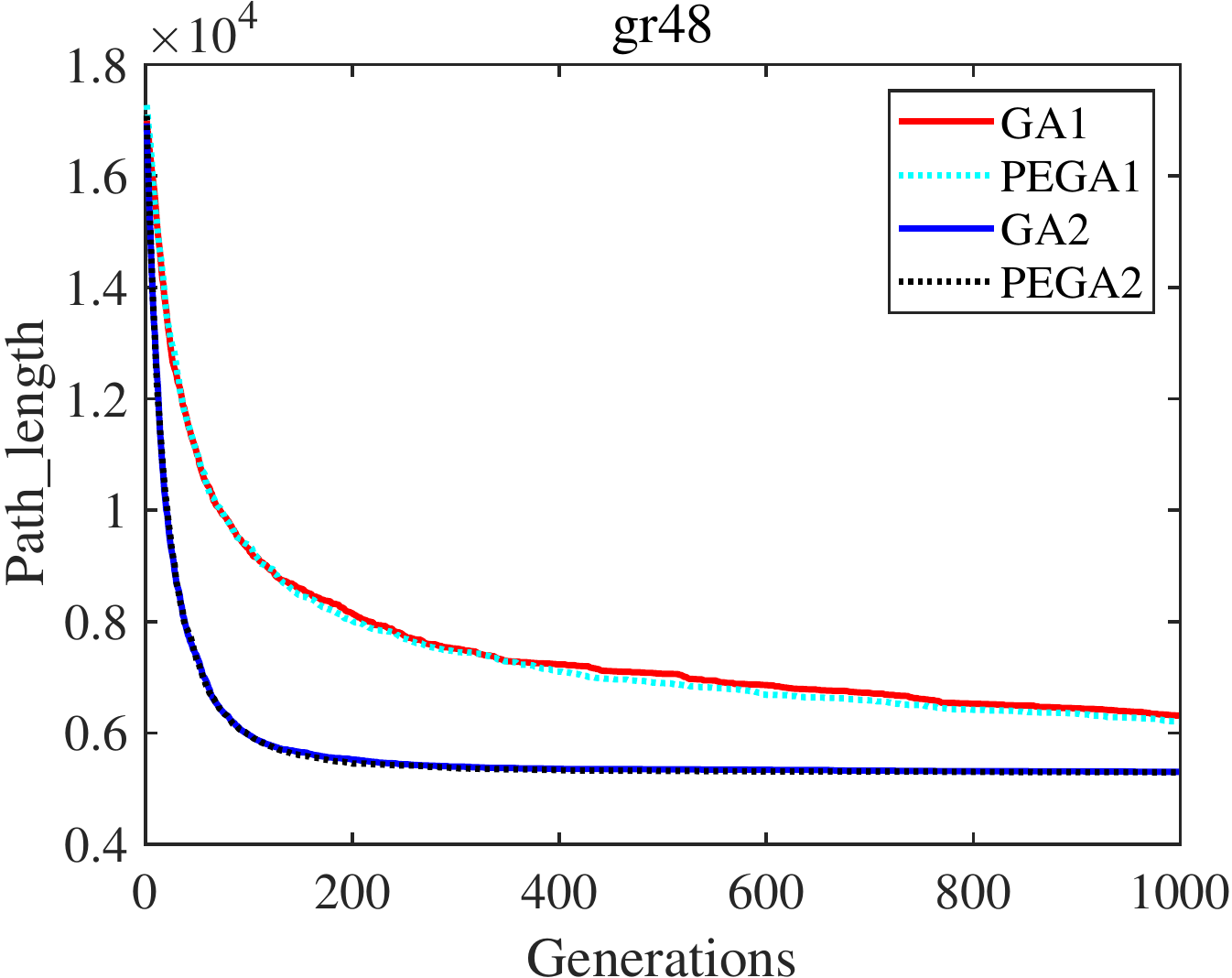}}
	\hspace{0.05in}
	\subfigure[rossover rate (0.1), mutation rate (0.15)]{
		\includegraphics[width=0.46\linewidth]{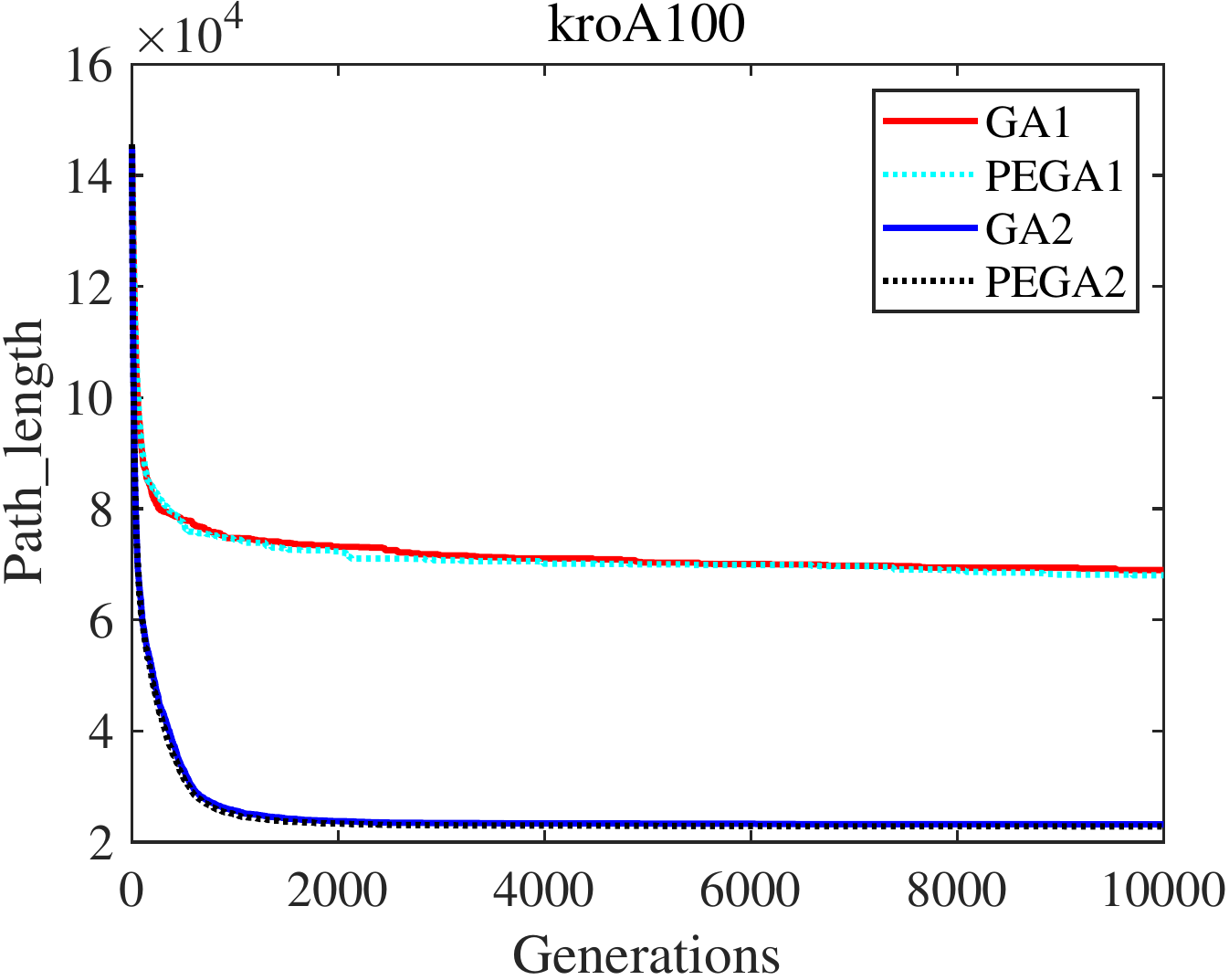}}
	\subfigure[rossover rate (0.1), mutation rate (0.15)]{
		\includegraphics[width=0.46\linewidth]{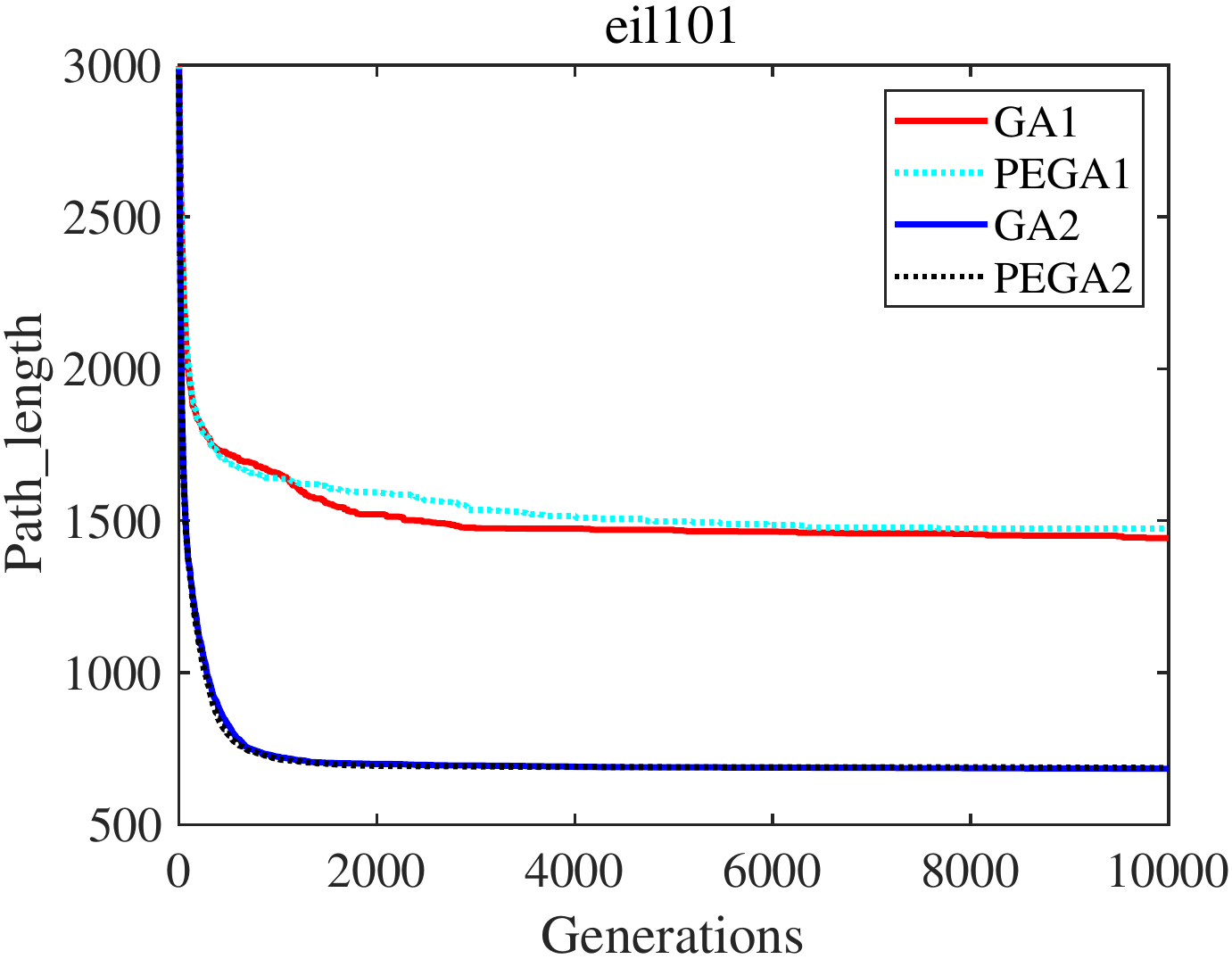}}
	\hspace{0.05in}
	\subfigure[rossover rate (0.1), mutation rate (0.15)]{
		\includegraphics[width=0.46\linewidth]{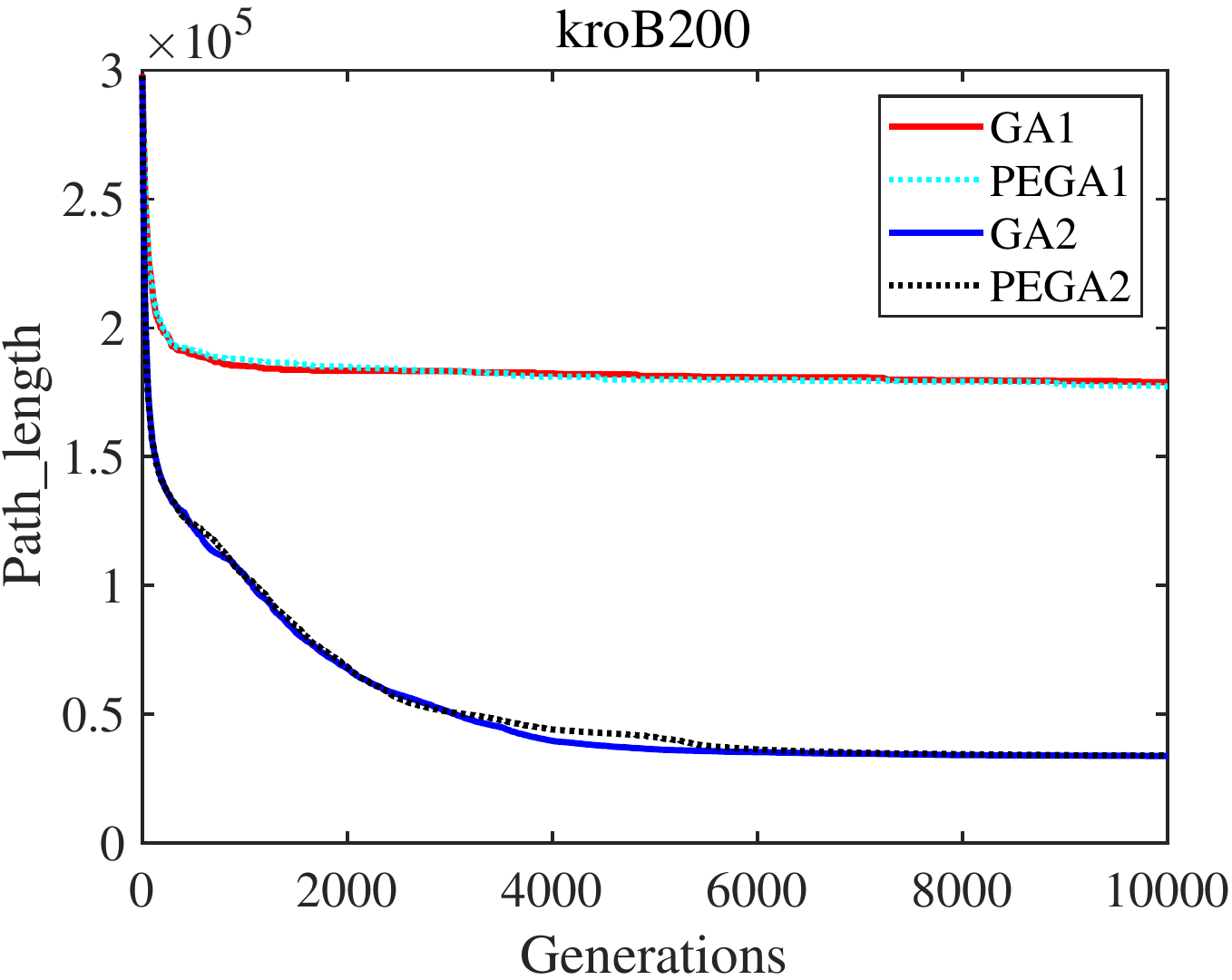}}
	\caption{Comparison of the convergence between PEGA and GA (The average result of 30 independent runs).}
	\label{fig:all}
\end{figure}

To further demonstrate the effectiveness of PEGA, we make PEGA1 and PEGA2 use the same random numbers with GA1 and GA2 to perform evolutionary operations, respectively. The experimental results are given in Fig. \ref{fig:pegavsga}. Magenta circle and cyan circle represent GA1 and GA2, respectively. Blue solid line and black solid line represent PEGA1 and PEGA2, respectively. From Fig. \ref{fig:pegavsga}, we see that PEGA1 and GA1 has the same convergence, and PEGA2 and GA2 has the same convergence, when the same random numbers are adopted. This is because our proposed secure computing protocols support exactly computations on encrypted data. In fact, Fig. \ref{fig:pegavsga} illustrates that PEGA is as effective as GA. In other words, given encrypted TSPs, PEGA can effectively approximate the optimal solution as GA.
\begin{figure}[ht]
	\centering
	\subfigure[rossover rate (0.08), mutation rate (0.1)]{
		\includegraphics[width=0.46\linewidth]{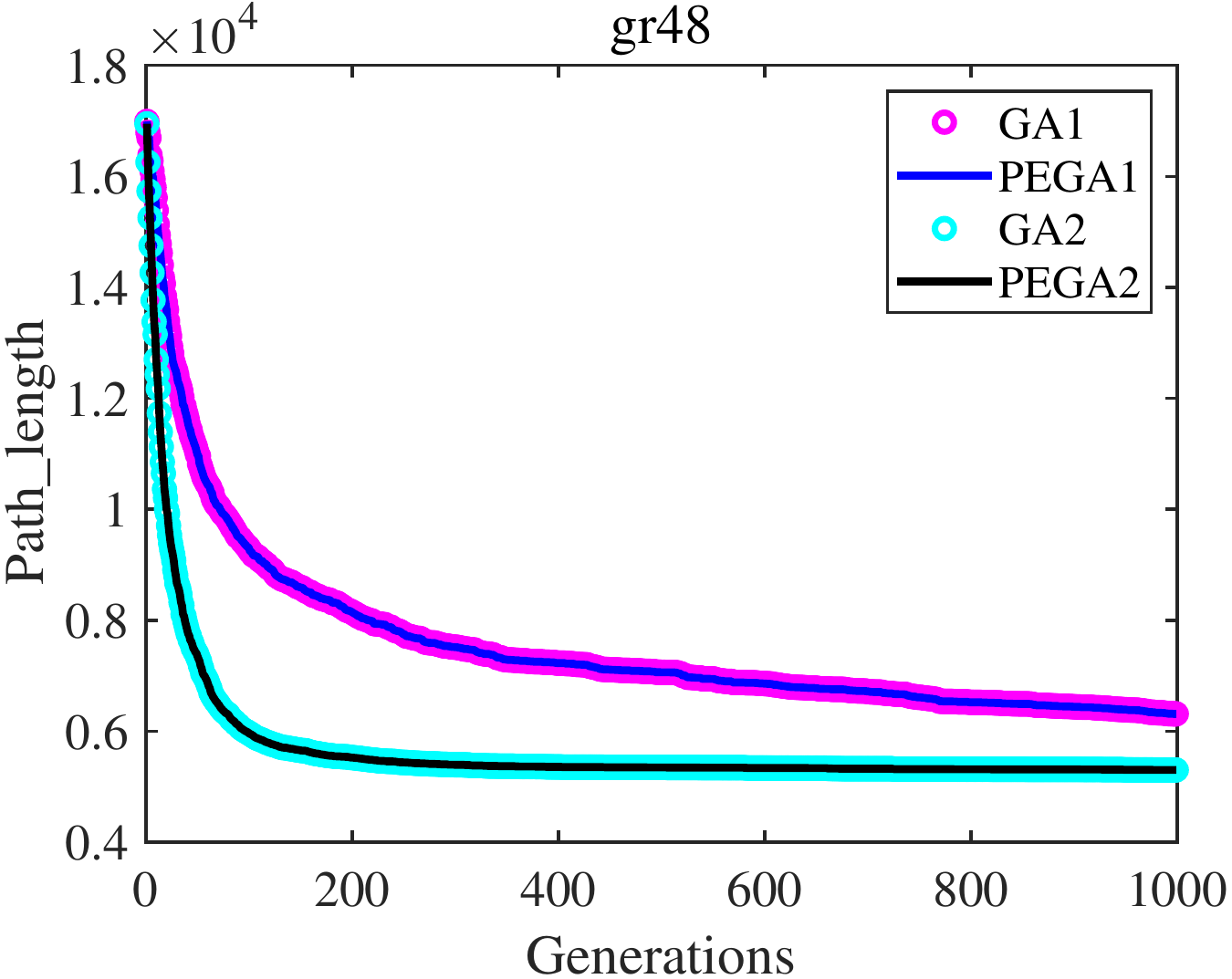}}
	\hspace{0.05in}
	\subfigure[rossover rate (0.1), mutation rate (0.15)]{
		\includegraphics[width=0.46\linewidth]{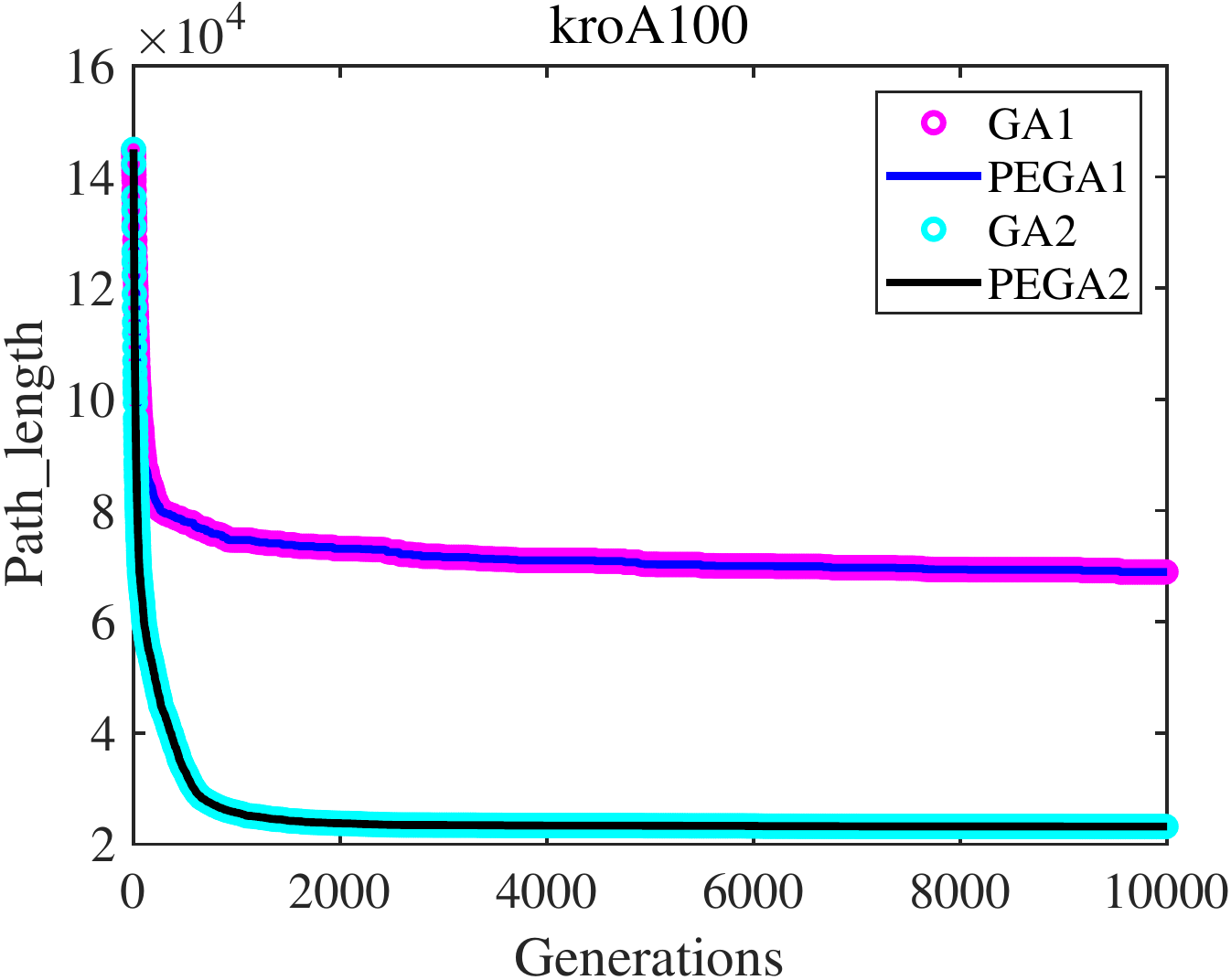}}
	\subfigure[rossover rate (0.1), mutation rate (0.15)]{
		\includegraphics[width=0.46\linewidth]{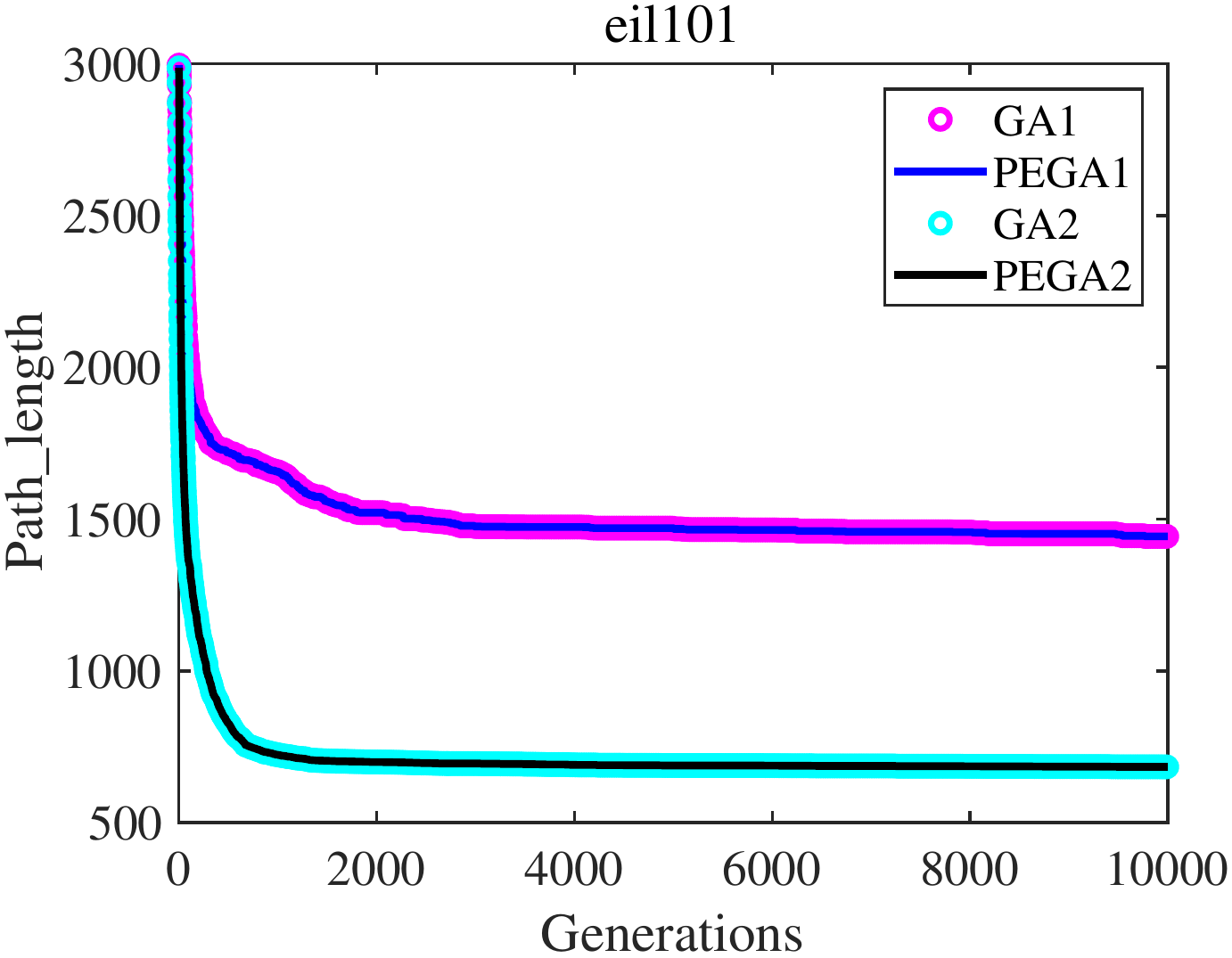}}
	\hspace{0.05in}
	\subfigure[rossover rate (0.1), mutation rate (0.15)]{
		\includegraphics[width=0.46\linewidth]{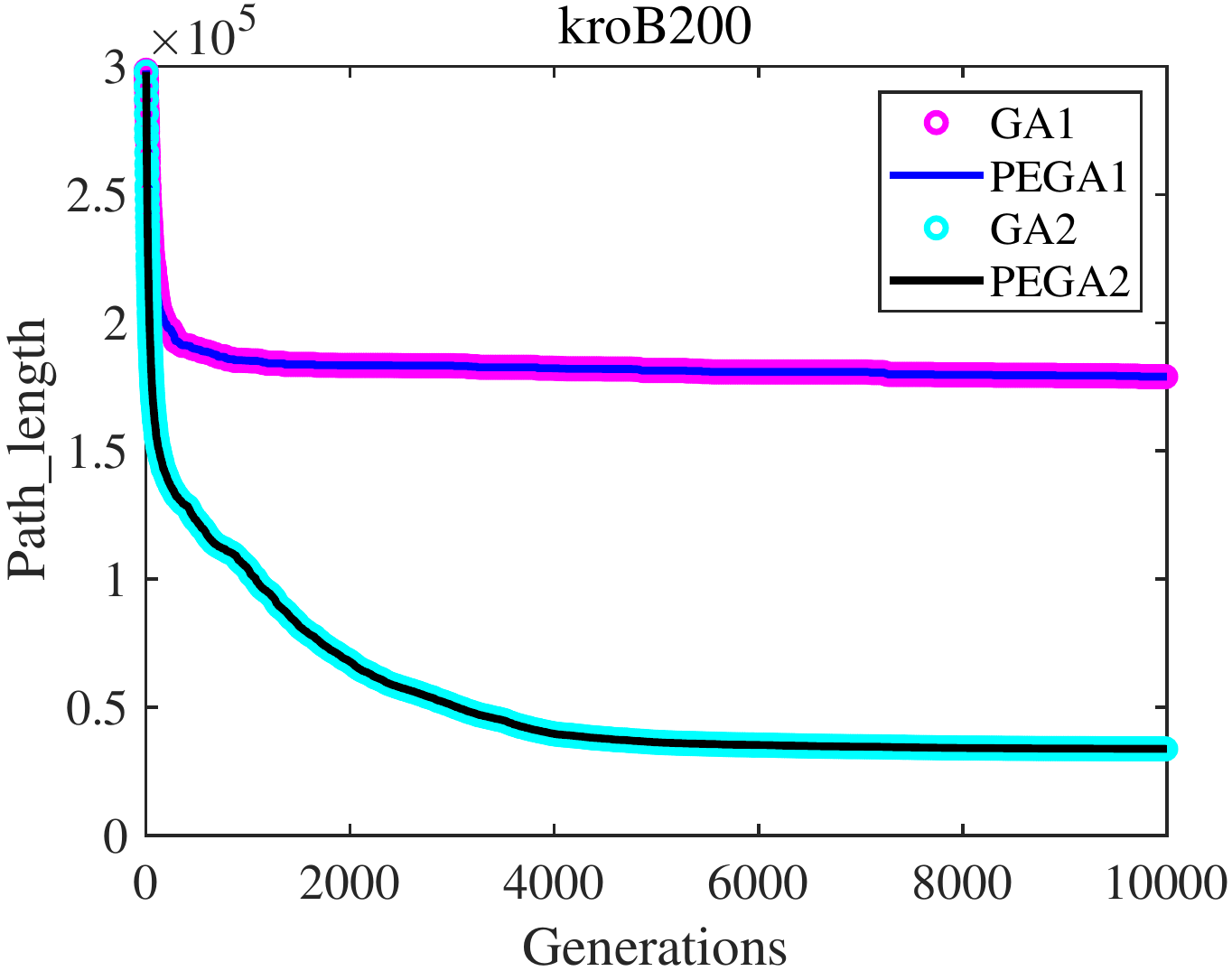}}
	\caption{Comparison of the convergence between PEGA and GA (The average result of 30 independent runs).}
	\label{fig:pegavsga}
\end{figure}

\subsection{Efficiency Evaluation}
In this section, we evaluate the efficiency of PEGA in terms of communication cost and computation cost. Table \ref{tab:cmm} shows the comparison results of communication cost between PEGA and GA. From Table \ref{tab:cmm}, we see that PEGA has a larger communication cost than GA. In PEGA, a user submits an encrypted TSP matrix, i.e., $\lb \mathbf{M}\rb$. On the contrary, the user in GA submits the TSP matrix $\mathbf{M}$ directly. As the ciphertext of THPC is significantly larger than its plaintext, the communication cost of PEGA is larger than that of GA. Also, when a small $N$ is set, it can significantly reduce the communication cost of PEGA. One possible explanation is that the smaller $N$, the smaller the ciphertext size of THPC is. As shown in Table \ref{tab:cmm}, we see that even a large TSP (e.g., kroB200) and a large $N$ are set, the communication cost of PEGA is less than 6 MB.
\begin{table}[!ht]
	\centering
	\caption{Comparison of Communication Cost between PEGA and GA (The average result of 30 independent runs)}
	\begin{threeparttable}
		\begin{tabular*}{0.48\textwidth}{@{}@{\extracolsep{\fill}}ccccc}
			\toprule
			& gr48 & kroA100 & eil101 & kroB200            \\ \midrule
			GA & 7 KB & 33 KB & 25 KB &  132 KB   \\ \hline
			PEGA$^\dagger$ & 173 KB & 0.74 MB & 0.76 MB & 2.98 MB \\ \hline
			PEGA$^\ddagger$ & 373 KB & 1.46 MB & 1.50 MB & 5.90 MB\\\bottomrule
		\end{tabular*}\label{tab:cmm}
		\begin{tablenotes}
			\scriptsize
			\item \textbf{Note.} $^\dagger$: the length of public key $N$ of Paillier cryptosystem in bits in PEGA is 128; $^\ddagger$: the length of public key $N$ of Paillier cryptosystem in bits in PEGA is 256;
		\end{tablenotes}
	\end{threeparttable}
\end{table}

Assume there be $n$ individuals and $m$ cities ($n>m$). A GA consists of \textsc{Gen\_Initial\_Pop}, \textsc{Evaluation}, \textsc{Selection}, \textsc{Crossover}, and \textsc{Mutation}. \textsc{Gen\_Initial\_Pop} initializes $n$ $m-$dimension individuals, so its computational complexity is $\mathcal{O}(mn)$. \textsc{Evaluation} is to compute each individual's route cost, so its computational complexity is also $\mathcal{O}(mn)$. \textsc{Selection} generally selects $n$ new individuals via a proportionate selection operator. The computational complexity of the conventional fitness proportionate selection operator is $\mathcal{O}(n^2)$. In this paper, PEGA adopts the idea of binary search to select new individuals, and its computational complexity is $\mathcal{O}(n\log n)$. Thus, PEGA improves the performance of \textsc{Selection} comparing to conventional GA. In this paper, we adopt ERX to perform \textsc{Crossover}. The computational complexity of ERX is $\mathcal{O}(m)$. The computational complexity of \textsc{Mutation} is $\mathcal{O}(n)$ for a population with $n$ individuals. Thus, we see that the computational complexity of GA is $\mathcal{O}(n^2t)+\mathcal{O}(mnt)$, while that of PEGA is $\mathcal{O}(n\log nt)+\mathcal{O}(mnt)$, where $t$ is the number of generations.

In contrary to GA, PEGA requires to encrypt the TSP matrix $\lb\mathbf{M}\rb$. Fig. \ref{fig:time} shows the runtime of encryption and searching for the optimal route of PEGA, where the runtime of searching is running one generation. From Fig. \ref{fig:time}, we can learn that the more cities, the more runtime of encryption for PEGA is. PEGA requires encrypting reachable routes between two cities. The more cities, the more reachable routes between two cities are. Also, PEGA takes around 3 s to produce a potential solution when $k$-tournament selection is adopted. Furthermore, for four TSPs, PEGA takes almost the same runtime to produce a possible solution. For PEGA, \textsc{Selection} performs most computations on encrypted data, and computations on encrypted data are time-consuming than computations on plaintext data. As the computational complexity of \textsc{Selection} of PEGA is $\mathcal{O}(n\log n)$, when $n$ is the same, PEGA takes almost the same runtime to produce a possible solution for four different TSPs. From Fig. \ref{fig:time}, we also see that the fitness proportionate selection consumes more runtime to generate a possible solution. One possible explanation is that the fitness proportionate selection operator requires more operations on encrypted data than $k$-tournament selection operator.
\begin{figure}
	\centering
	\includegraphics[width=0.56\linewidth]{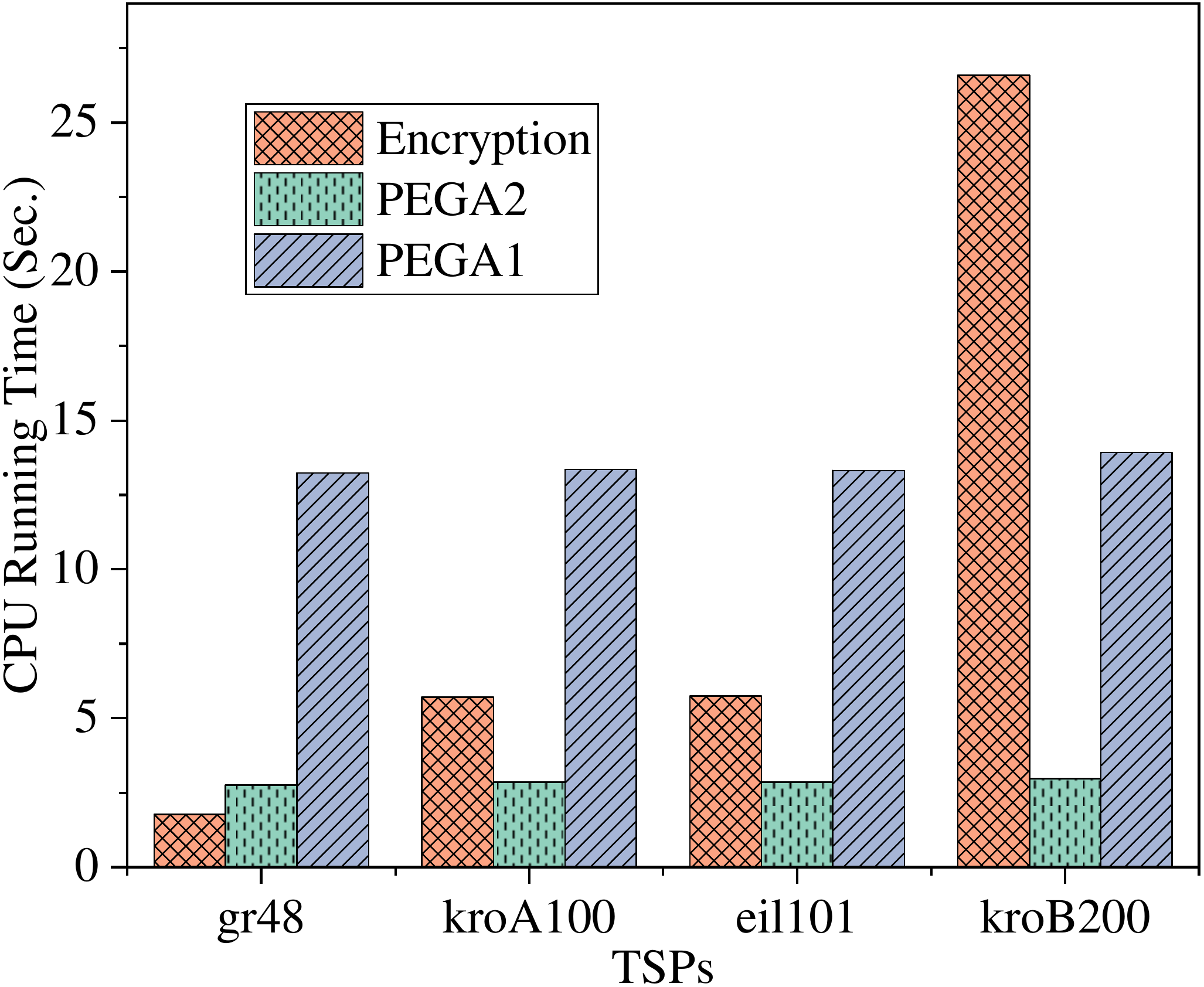}
	\caption{Runtime of PEGA (The average result of 30 independent runs).}
	\label{fig:time}
\end{figure}

\section{Conclusion}
In this paper, we proposed the computing paradigm of evolution as a service (EaaS) and designed a privacy-preserving genetic algorithm for COPs based on EaaS, called PEGA. To show the effectiveness and efficiency of PEGA, we use the widely known TSP to evaluate PEGA. In PEGA, a user encrypts her TSP matrix to protect the privacy and outsources the evolutionary computations to cloud servers. The cloud server performs evolutionary computations over encrypted data and produces an effective solution as conventional GA. To support operations on encrypted TSPs, this paper presented a secure division protocol (\texttt{SecDiv}) and a secure comparison protocol (\texttt{SecCmp}) falling in the twin-server architecture. Experimental evaluations on four TSPs (i.e., gr48, KroA100, eil101, and KroB200) show that there is no significant difference between PEGA and conventional GA. Also, given encrypted TSPs, PEGA with $k$-tournament selection operator can produce one potential solution around 3 s. For future work, we will extend the idea of EaaS to other algorithms, such as particle swarm optimization (PSO), ant colony optimization (ACO).

\ifCLASSOPTIONcaptionsoff
\newpage
\fi

\bibliographystyle{IEEEtran}
\bibliography{IEEEabrv,ref}
	
\begin{IEEEbiography}[{\includegraphics[width=1in,height=1.25in,clip,keepaspectratio]{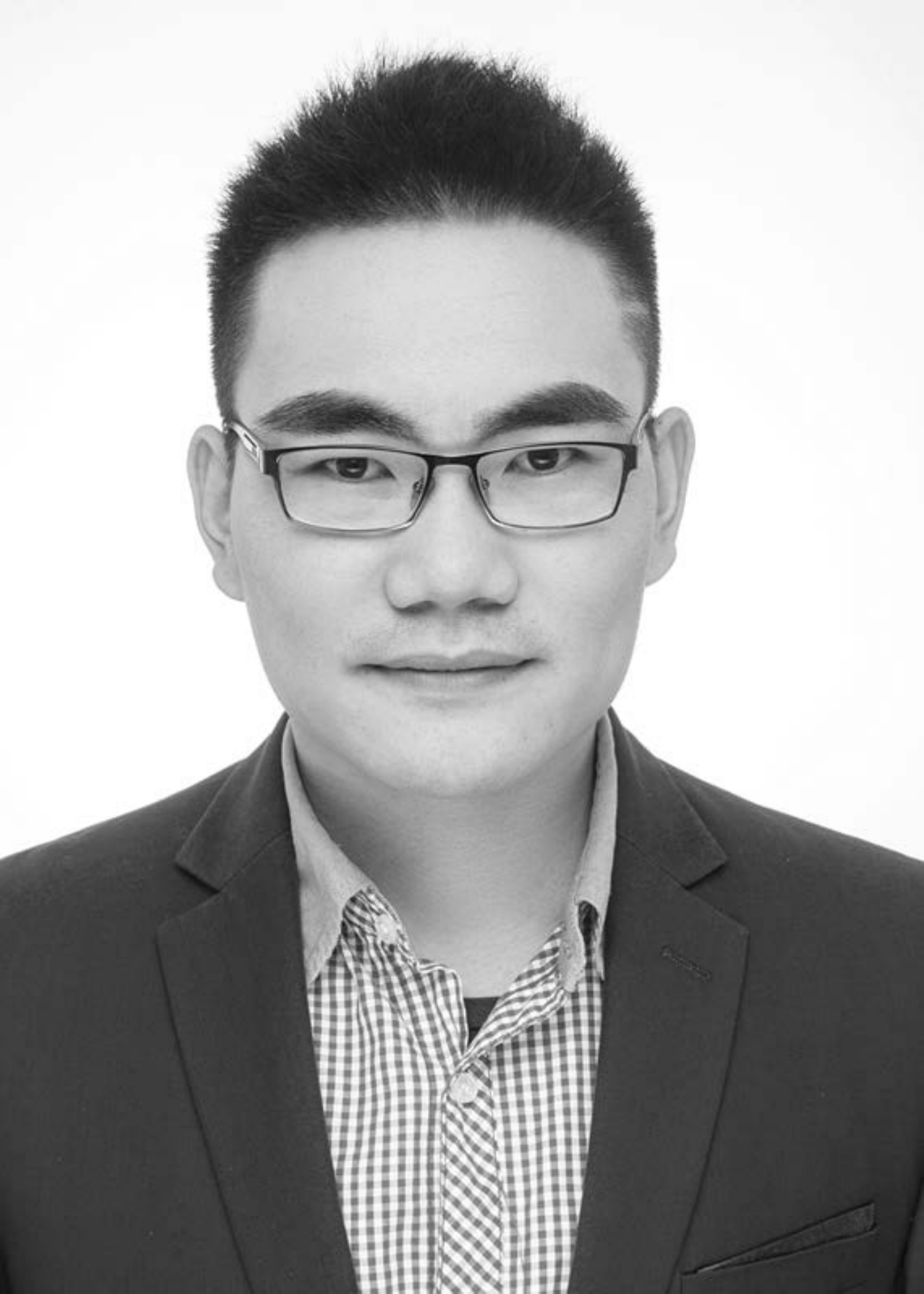}}]{Bowen Zhao} (M'22) received his Ph.D. degree in cyberspace security from South China University of Technology, China, in 2020. He was a research scientist in the school of computing and information systems, Singapore Management University, from 2020 to 2021. Now, he is an associate professor at Guangzhou Institute of Technology, Xidian University, Guangzhou, China. His current research interests include privacy-preserving computation and learning and privacy-preserving crowdsensing.
\end{IEEEbiography}
	
\begin{IEEEbiography}[{\includegraphics[width=1in,height=1.25in,clip,keepaspectratio]{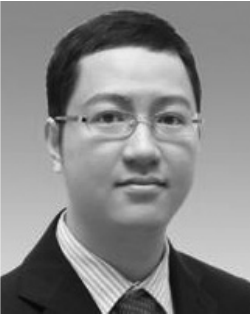}}]{Wei-Neng Chen} (S'07–M'12–SM'17) received thebachelor's and Ph.D. degrees from Sun Yat-sen University, Guangzhou, China, in 2006 and 2012, respectively.
		
He is currently a Professor with the School of Computer Science and Engineering, South China University of Technology, Guangzhou. He has co-authored over 90 papers in international journals and conferences, including over 30 papers in IEEE TRANSACTIONS journals. His current research interests include swarm intelligence algorithms and their applications on cloud computing, operations research, and software engineering.
		
Dr. Chen was a recipient of the National Science Fund for Excellent Young Scholars in 2016 and the IEEE Computational Intelligence Society Outstanding Dissertation Award in 2016 for his doctoral thesis. He is 	the Vice-Chair of the IEEE Guangzhou Section.
\end{IEEEbiography}

\begin{IEEEbiography}[{\includegraphics[width=1in,height=1.25in,clip,keepaspectratio]{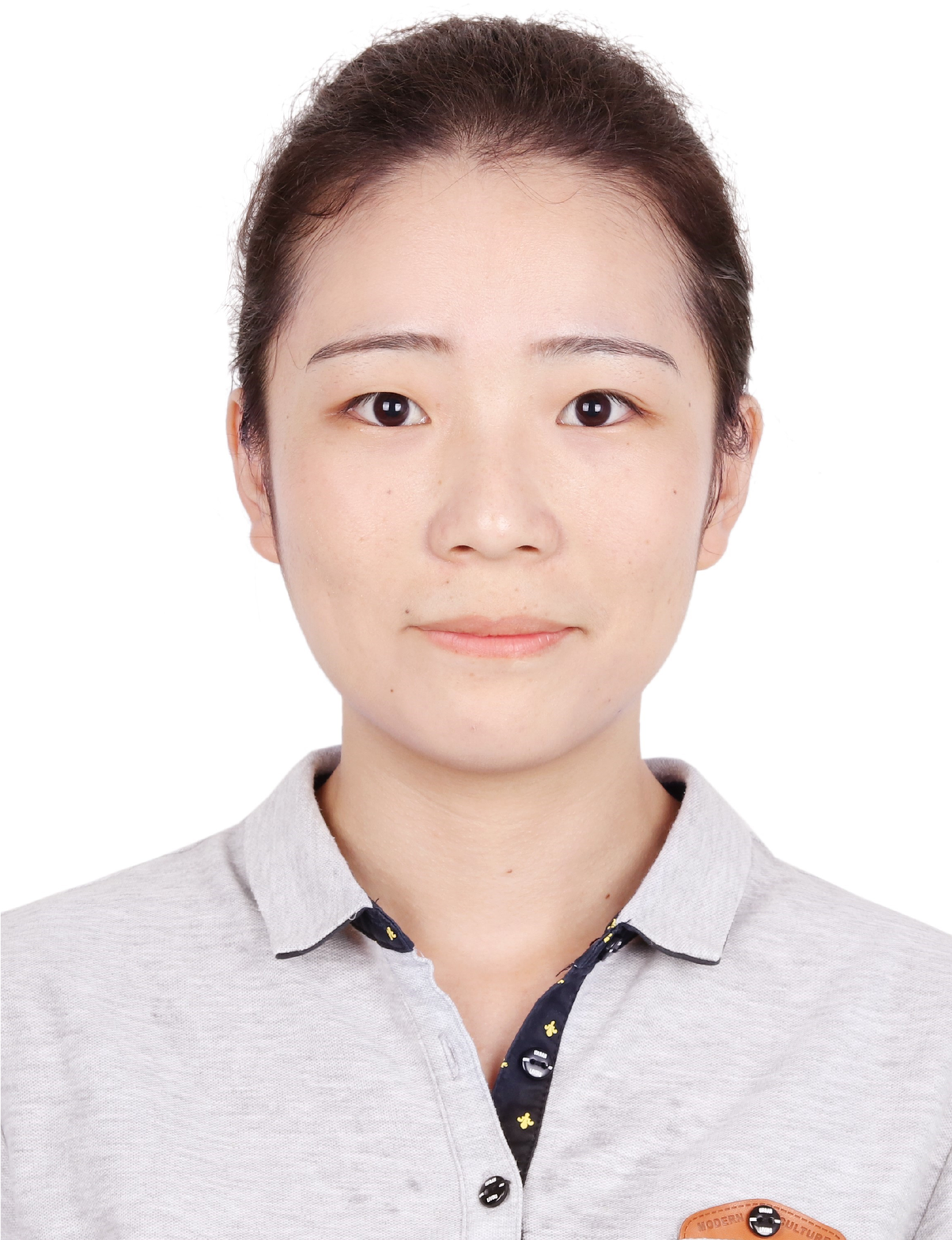}}]{Feng-Feng Wei} received her Bachelor’s degree in computer science from South China University of Technology, Guangzhou, China, in 2019, where she is currently pursuing the Ph.D degree. Her current research interests include evolutionary computation algorithms and their applications on expensive and distributed optimization in real-world problems.
\end{IEEEbiography}

\begin{IEEEbiography}[{\includegraphics[width=1in,height=1.25in,clip,keepaspectratio]{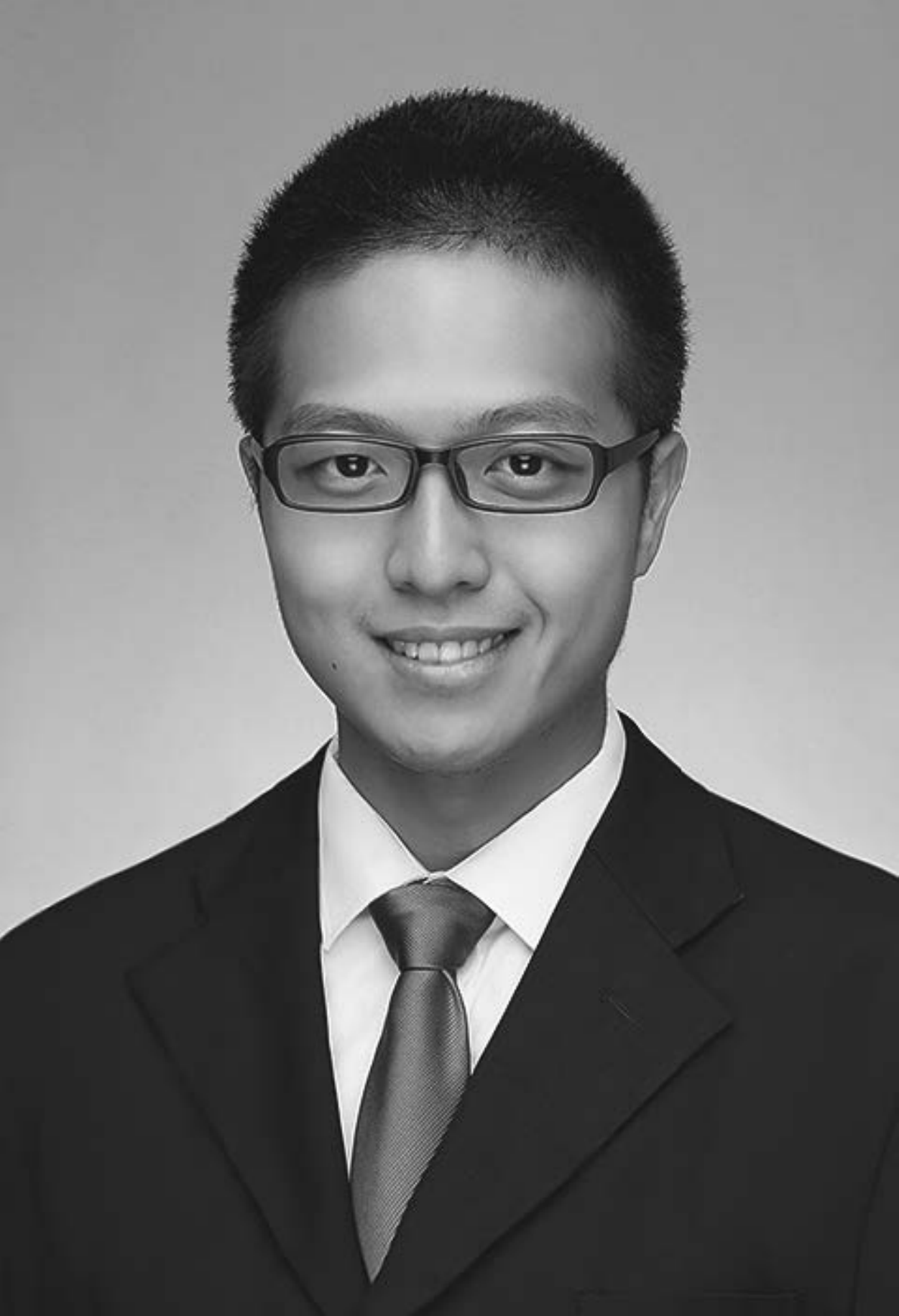}}]{Ximeng Liu} (S'13-M'16-SM'21) received the B.Sc. degree in electronic engineering from Xidian University, Xi'an, China, in 2010 and the Ph.D. degree in Cryptography from Xidian University, China, in 2015. Now he is the full professor in the College of Computer Science and Data Science, Fuzhou University. Also, he was a research fellow at Peng Cheng Laboratory, Shenzhen, China. He has published more than 200 papers on the topics of cloud security and big data security including papers in IEEE TOC, IEEE TII, IEEE TDSC, IEEE TSC, IEEE IoT Journal, and so on. He awards “Minjiang Scholars” Distinguished Professor, “Qishan Scholars” in Fuzhou University, and ACM SIGSAC China Rising Star Award (2018). His research interests include cloud security, applied cryptography and big data security.
\end{IEEEbiography}

\begin{IEEEbiography}[{\includegraphics[width=1in,height=1.25in,clip,keepaspectratio]{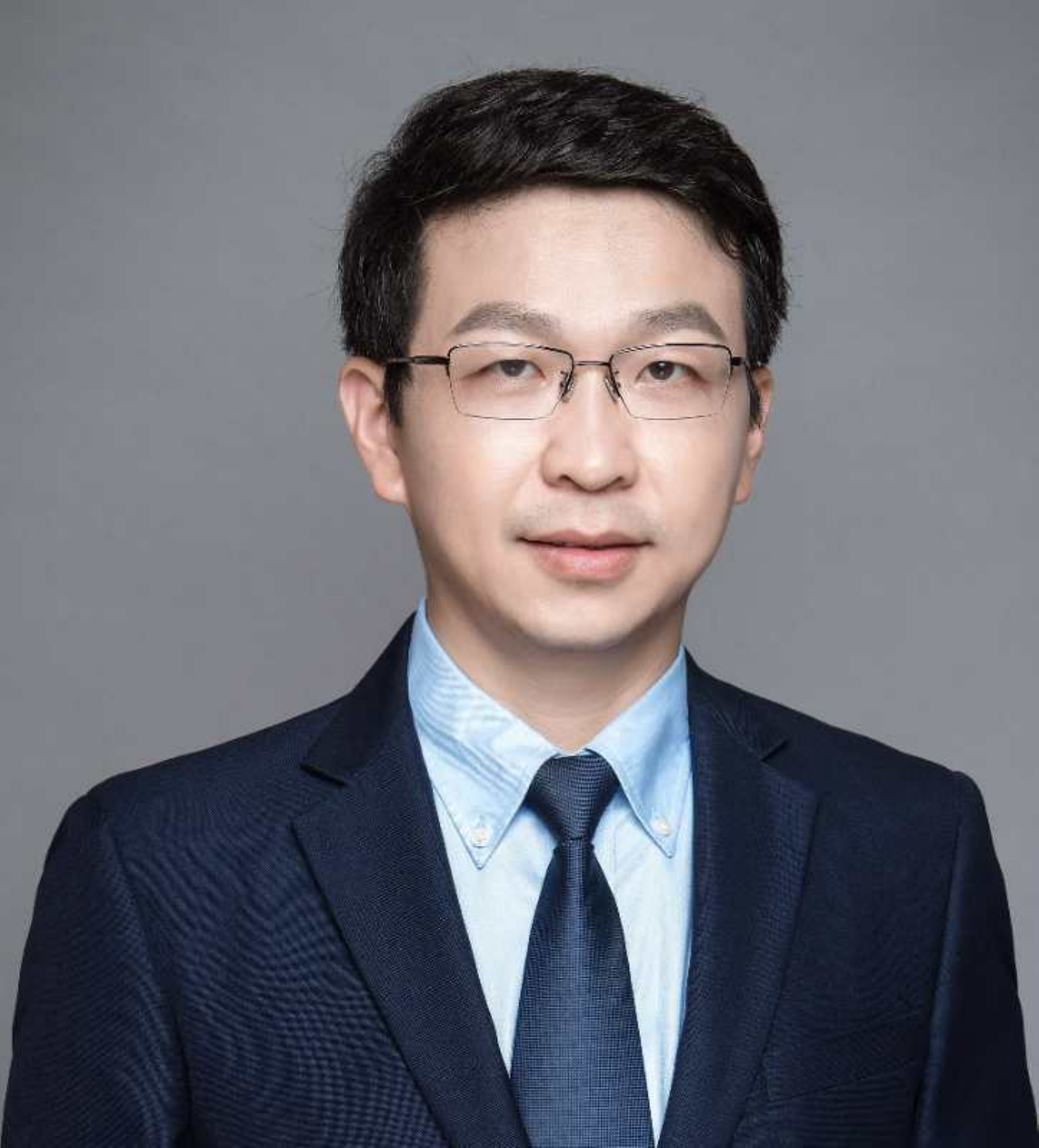}}]{Qingqi Pei} (SM'22) received his B.S., M.S. and Ph.D. degrees in Computer Science and Cryptography from Xidian University, in 1998, 2005 and 2008, respectively. He is now a Professor and member of the State Key Laboratory of Integrated Services Networks, also a Professional Member of ACM and Senior Member of IEEE, Senior Member of Chinese Institute of Electronics and China Computer Federation. His research interests focus on digital contents protection and wireless networks and security. 
\end{IEEEbiography}
	
\begin{IEEEbiography}[{\includegraphics[width=1in,height=1.25in,clip,keepaspectratio]{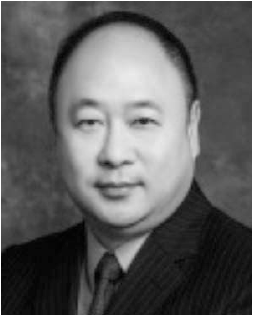}}]{Jun Zhang} (M'02–SM'08–F'17) received the Ph.D. degree from the City University of Hong Kong, Hong Kong, in 2002. He is currently a Professor, Hanyang University, South Korea. His current research interests include computational intelligence, cloud computing, high-performance computing, operations research, and power electronic circuits. Dr. Zhang was a recipient of the Changjiang Chair Professor from the Ministry of Education, China, in 2013, the China National Funds for Distinguished Young Scientists from the National Natural Science Foundation of China in 2011, and the First-Grade Award in Natural Science Research from the Ministry of Education, China, in 2009. He is currently an Associate Editor of the IEEE TRANSACTIONS ON EVOLUTIONARY COMPUTATION, the IEEE TRANSACTIONS ON CYBERNETICS, and the IEEE TRANSACTIONS ON INDUSTRIAL ELECTRONICS. 
\end{IEEEbiography}
	
%\begin{IEEEbiography}[{\includegraphics[width=1in,height=1.25in,clip,keepaspectratio]{deng.pdf}}]{Robert H. Deng} (Fellow, IEEE) is currently the	AXA Chair Professor of cybersecurity, the Director of the Secure Mobile Centre, and the Deputy Dean of the Faculty \& Research, School of Information Systems, Singapore Management University (SMU). His research interests are in the areas of data security and privacy, network security, and system security. He is a Fellow of Academy of Engineering Singapore. He received the Outstanding University Researcher Award from the National University of Singapore, Lee Kuan Yew Fellowship for Research Excellence from SMU, and Asia–Pacific Information Security Leadership Achievements Community Service Star from International Information Systems Security Certification Consortium. He has served the Steering Committee Chair of the ACM Asia Conference on Computer and Communications Security. He serves/served on many editorial boards and conference committees, including the editorial boards of ACM Transactions on Privacy and Security, IEEE SECURITY AND PRIVACY, the IEEE TRANSACTIONS ON DEPENDABLE AND SECURE COMPUTING, the IEEE TRANSACTIONS ON INFORMATION FORENSICS AND SECURITY, and the Journal of Computer Science and Technology. 
%\end{IEEEbiography}
\end{document}